\documentclass[a4paper,UKenglish,numberwithinsect,cleveref, autoref, thm-restate]{lipics-v2021}

\usepackage{color}

\newcommand{\R}{\mathbb{R}}

\newcommand{\Geqt}{G_{\Delta}}

\usepackage{todonotes}

\usepackage{gensymb}

\pdfoutput=1 
\hideLIPIcs  

\title{Reconfiguration and Locomotion with Joint Movements in the Amoebot Model} 

\author{Andreas Padalkin}{Paderborn University, Germany}{andreas.padalkin@upb.de}{https://orcid.org/0000-0002-4601-9597}{}

\author{Manish Kumar}{New Jersey Institute of Technology, US}{manish.kumar@njit.edu}{https://orcid.org/0000-0003-0620-3303}{}

\author{Christian Scheideler}{Paderborn University, Germany}{scheideler@upb.de}{https://orcid.org/0000-0002-5278-528X}{}

\authorrunning{A. Padalkin, M. Kumar, and C. Scheideler} 

\Copyright{Andreas Padalkin, Manish Kumar, and Christian Scheideler} 



\category{} 

\relatedversion{} 

\funding{This work has been supported by the DFG Project SCHE 1592/10-1 and the Israel Science Foundation under Grants 867/19 and 554/23.}

\nolinenumbers 

\EventEditors{John Q. Open and Joan R. Access}
\EventNoEds{2}
\EventLongTitle{42nd Conference on Very Important Topics (2022)}
\EventShortTitle{XYZ 2022}
\EventAcronym{XYZ}
\EventYear{2016}
\EventDate{December 24--27, 2016}
\EventLocation{Little Whinging, United Kingdom}
\EventLogo{}
\SeriesVolume{42}
\ArticleNo{23}

\begin{document}

\maketitle

\begin{abstract}
    We are considering the geometric amoebot model where a set of $n$ \emph{amoebots} is placed on the triangular grid. An amoebot is able to send information to its neighbors, and to move via expansions and contractions. Since amoebots and information can only travel node by node, most problems have a natural lower bound of $\Omega(D)$ where $D$ denotes the diameter of the structure. Inspired by the nervous and muscular system, Feldmann et al.\ have proposed the \emph{reconfigurable circuit extension} and the \emph{joint movement extension} of the amoebot model with the goal of breaking this lower bound.

    In the joint movement extension, the way amoebots move is altered. Amoebots become able to push and pull other amoebots. Feldmann et al.\ demonstrated the power of joint movements by transforming a line of amoebots into a rhombus within $O(\log n)$ rounds. However, they left the details of the extension open. The goal of this paper is therefore to formalize and extend the joint movement extension. In order to provide a proof of concept for the extension, we consider two fundamental problems of modular robot systems: \emph{reconfiguration} and \emph{locomotion}.

    We approach these problems by defining meta-modules of rhombical and hexagonal shape, respectively. The meta-modules are capable of movement primitives like sliding, rotating, and tunneling. This allows us to simulate reconfiguration algorithms of various modular robot systems. Finally, we construct three amoebot structures capable of locomotion by rolling, crawling, and walking, respectively.
\end{abstract}




\section{Introduction}

Programmable matter consists of homogeneous nano-robots that are able to change the properties of the matter in a programmable fashion, e.g., the shape, the color, or the density \cite{DBLP:journals/ijhsc/ToffoliM93}.
We are considering the geometric amoebot model \cite{DBLP:series/lncs/DaymudeHRS19,DBLP:conf/wdag/DaymudeRS21,DBLP:conf/spaa/DerakhshandehDGRSS14} where a set of $n$ nano-robots (called \emph{amoebots}) is placed on the triangular grid.
An amoebot is able to send information to its neighbors, and to move by first expanding into an unoccupied adjacent node, and then contracting into that node.
Since amoebots and information can only travel node by node, most problems have a natural lower bound of $\Omega(D)$ where $D$ denotes the diameter of the structure.
Inspired by the \emph{nervous} and \emph{muscular system}, Feldmann et al.~\cite{DBLP:journals/jcb/FeldmannPSD22} proposed the \emph{reconfigurable circuit extension} and the \emph{joint movement extension} with the goal of breaking this lower bound.

In the reconfigurable circuit extension, the amoebot structure is able to interconnect amoebots by \emph{circuits}.
Each amoebot can send a primitive signal on circuits it is connected to.
The signal is received by all amoebots connected to the same circuit.
Among others, Feldmann et al.~\cite{DBLP:journals/jcb/FeldmannPSD22} solved the leader election problem, compass alignment problem, and chirality agreement problem within $O(\log n)$ rounds w.h.p.\footnote{An event holds with high probability (w.h.p.) if it holds with probability at least $1 - 1/n^c$, where the constant $c$ can
be made arbitrarily large.}
These problems will be important to coordinate the joint movements.
Afterwards, Padalkin et al.~\cite{DBLP:conf/podc/PadalkinS24,DBLP:journals/nc/PadalkinSW24} proposed polylogarithmic time algorithms to various problems including spanning tree computation, symmetry detection, and shortest path computation.

In the joint movement extension, the way amoebots move is altered.
In a nutshell,
an expanding amoebot is capable of pushing other amoebots away from it,
and a contracting amoebot is capable of pulling other amoebots towards it.
Feldmann et al.~\cite{DBLP:journals/jcb/FeldmannPSD22} demonstrated the power of joint movements by transforming a line of amoebots into a rhombus within $O(\log n)$ rounds.
However, they left the details of the extension open.
The goal of this paper is therefore to formalize and extend the joint movement extension.
In order to provide a proof of concept for the extension, we consider two fundamental problems of modular robot systems (MRS): reconfiguration and locomotion.
We study these problems from a centralized view to explore the limits of the extension.

In the reconfiguration problem, an MRS has to reconfigure its structure into a given shape.
Examples for reconfiguration algorithms in the original amoebot model can be found in \cite{DBLP:conf/nanocom/DerakhshandehGR15,DBLP:conf/spaa/DerakhshandehGR16,DBLP:conf/dna/KostitsynaSW22,DBLP:journals/dc/LunaFSVY20}.
However, all of these are subject of the aforementioned natural lower bound.
To our knowledge, polylogarithmic time solutions were found for two types of MRSs: in the nubot model \cite{DBLP:conf/innovations/WoodsCGDWY13} and crystalline atom model \cite{DBLP:conf/isaac/AloupisCDLAW08}.
We will show that in the joint movement extension, the amoebots are able to simulate the reconfiguration algorithm for the crystalline atom model, and with that break the lower bound of the original amoebot model.

In the locomotion problem, an MRS has to move along an even surface as fast as possible.
We might also ask the MRS to transport an object along the way.
In the original amoebot model, one would use the spanning tree primitive to move along the surface \cite{DBLP:series/lncs/DaymudeHRS19}.
However, we only obtain a constant velocity with that.
Furthermore, the original amoebot model does not allow us to transport any objects.
In terrestrial environments, there are three basic types of locomotion: rolling, crawling, and walking \cite{hirose1991three,DBLP:conf/embc/KehoeP19}.
For each of these locomotion types, we will present an amoebot structure.

\section{Preliminaries}

In this section, we introduce the geometric amoebot model \cite{DBLP:conf/wdag/DaymudeRS21,DBLP:conf/spaa/DerakhshandehDGRSS14} and formalize the joint movement extension.


\subsection{Geometric Amoebot Model}
\label{sec:model:amoebot}

In this section, we introduce the \emph{geometric amoebot model} \cite{DBLP:conf/wdag/DaymudeRS21}.
We slightly deviate from the original model to make it more suitable to our extension.
A set of $n$ amoebots is placed on the infinite regular triangular grid graph $\Geqt = (V_\Delta, E_\Delta)$ (see \Cref{fig:model:classic}).
An amoebot is an anonymous, randomized finite state machine in the form of a line segment.
The endpoints may either occupy the same node or two adjacent nodes.
If the endpoints occupy the same node, the amoebot has length $0$ and is called \emph{contracted} and otherwise, it has length $1$ and is called \emph{expanded}.
Every node of $\Geqt$ is occupied by at most one amoebot.
Two endpoints of different amoebots that occupy adjacent nodes in $\Geqt$ are connected by \emph{bonds} (red edges).
Let $V \subseteq V_\Delta$ be the nodes occupied by the amoebots.
We define graph $S = \Geqt|_V$ induced by $V$ to be the \emph{connectivity graph} of the amoebot structure.
Initially, $S$ is connected and we require that $S$ stays connected at all times.

\begin{figure}[tb]
    \centering
    \includegraphics[page=1]{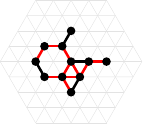}
    \caption{
        Amoebot structure.
        The amoebots are shown in black.
        Red edges indicate bonds.
    }
    \label{fig:model:classic}
\end{figure}

An amoebot can move through expansions and contractions.
A contracted amoebot occupying node $u$ can \emph{expand} into an unoccupied adjacent node $v$ (see the blue amoebot in \Cref{fig:classic:expansion}).
Thereafter, it occupies both, $u$ and $v$.
An expanded amoebot occupying nodes $v$ and $w$ can \emph{contract} into one of its occupied nodes (see the green amoebot in \Cref{fig:classic:contraction}).
Thereafter, it only occupies that node.
Furthermore, let a contracted amoebot $x$ occupy node $u$ and an expanded amoebot $y$ occupy nodes $v$ and $w$ such that $u$ and $v$ are adjacent.
These amoebots can perform a \emph{handover} where $x$ expands into $v$ and $y$ contracts into $w$ simultaneously (see \Cref{fig:classic:handover} where $x$ is marked in blue and $y$ is marked in green).
Thereafter, $x$ occupies $u$ and $v$, and $y$ occupies $w$.

We refer to \cite{DBLP:conf/wdag/DaymudeRS21} for more details.

\begin{figure}[tb]
    \centering
    \begin{subfigure}[b]{.3\textwidth}
        \centering
        \includegraphics[page=1]{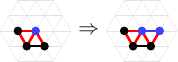}
        \caption{Expansion.}
        \label{fig:classic:expansion}
    \end{subfigure}
    \hfill
    \begin{subfigure}[b]{.3\textwidth}
        \centering
        \includegraphics[page=1]{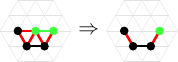}
        \caption{Contraction.}
        \label{fig:classic:contraction}
    \end{subfigure}
    \hfill
    \begin{subfigure}[b]{.3\textwidth}
        \centering
        \includegraphics[page=1]{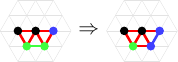}
        \caption{Handover.}
        \label{fig:classic:handover}
    \end{subfigure}
    \caption{
        Movements in the classical model.
        Red edges indicate bonds.
        Blue amoebots are expanding.
        Green amoebots are contracting.
    }
    \label{fig:classic:movements}
\end{figure}



\subsection{Joint Movement Extension}

In the \emph{joint movement extension} \cite{DBLP:journals/jcb/FeldmannPSD22}, the way the amoebots move is altered.
The idea behind the extension is to allow amoebots to push and pull other amoebots.
In the following, we formalize the joint movement extension.


We first redefine the connectivity graph of an amoebot structure.
In the graph, we represent each amoebot and bond by an edge.
Let $A$ denote the set of all edges representing amoebots, $V$ the set of all their endpoints, and $B$ the set of edges representing bonds between these endpoints.
The edges of $A$ are pairwise disjoint.
We define the \emph{connectivity graph} of an amoebot structure as the graph $S = (V, A \cup B)$.

Note that the connectivity graph does not contain any information about the positions of the amoebots in the amoebot structure.
Each amoebot contributes one edge to $A$ and two nodes to $V$ even if it is in a contracted state.

To each edge in $A \cup B$, we assign a geometric orientation and a length, such that we obtain line segments.
We only allow orientations in parallel to the axes of the triangular grid $\Geqt$.
Each amoebot has a length in $[0,1]$ and each bond has a length of $1$.
We consider two edges in $A \cup B$ as adjacent if and only if they share a common endpoint or they are connected by another edge of length $0$ (bonds connected by a contracted amoebot).
The \emph{amoebot structure} is defined by its connectivity graph and the assignment of orientations and lengths.
For the sake of simplicity, we will denote each amoebot structure by its connectivity graph.


We call an amoebot structure \emph{connected} if and only if $S$ is connected.
We call an amoebot structure \emph{valid} if and only if we can embed the graph in the plane $\R^2$ in compliance with the orientations and lengths of all amoebots and bonds (i.e., $A \cup B$) such that only adjacent edges intersect.
Note that an embedding is unique except for translations if and only if it is connected.
We call a valid amoebot structure \emph{triangular} if and only if the embedding can be aligned with the triangular grid.
Note that this is the case if and only if each amoebot in $A$ has either length $0$ or $1$.
We call a valid amoebot structure \emph{complete} if and only if there is no other valid amoebot structure $S' = (V, A \cup B')$ with $B \subsetneq B'$.
We assume that initially, the amoebot structure $S_0 = (V, A \cup B_0)$ is connected, valid, triangular, and complete.


We assume the fully synchronous activation model, i.e., the time is divided into synchronous rounds, and every amoebot is active in each round.
Furthermore, we make the idealistic assumption that all movements start at the same time and are performed at the same speed.
W.l.o.g., we assume that all movements occur within the time period $[0,1]$.



Let $S_i = (V, A \cup B_i)$ be an connected, valid, triangular, and complete amoebot structure.
Joint movements are performed in four steps.
In the first step,
the amoebots remove bonds from $S_i$ as follows.
Each amoebot can decide to release an arbitrary subset of its currently incident bonds in $S_i$.
A bond is removed if and only if either of the amoebots at the endpoints releases the bond.
Note that we only remove bonds (i.e., edges in $B_i$) but no amoebots (i.e., edges in $A$).
Let $R_i \subseteq B_i$ be the set of the remaining bonds and $S^0_i = (V, A \cup R_i)$ the resulting amoebot structure.

We require that $S^0_i$ is connected since otherwise, disconnected parts might float apart.
We say that a \emph{connectivity conflict} occurs if and only if $S^0_i$ is not connected.
Whenever a connectivity conflict occurs, the amoebot structure transitions into an undefined state such that we become unable to make any statements about the structure.
Note that $S^0_i$ is still valid and triangular since $S^0_i$ is a subgraph of $S_i$.
However, $S^0_i$ is not complete anymore if at least one bond was removed.



In the second step, each amoebot may perform one of the following movements.
A contracted amoebot may expand on one of the axes as follows (see blue amoebot in \Cref{fig:expansion}).
At $t = 0$,
the amoebot can reorientate itself and reassign each of its incident bonds to one of its endpoints.
Bonds assigned to an endpoint will stay with that endpoint as the amoebot expands.
At $t \in [0,1]$,
the amoebots has a length of $t$.
In the process,
the incident bonds do not change their orientations or lengths.
As a result, the expanding amoebot pushes all connected amoebots.
An expanded amoebot may contract analogously by reversing the contraction (see green amoebot in \Cref{fig:contraction}).
Thereby, it pulls all connected amoebots.

\begin{figure}[tb]
    \centering
    \begin{subfigure}[b]{.45\textwidth}
        \centering
        \includegraphics[page=1]{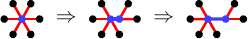}
        \caption{Expansion.}
        \label{fig:expansion}
    \end{subfigure}
    \hfill
    \begin{subfigure}[b]{.45\textwidth}
        \centering
        \includegraphics[page=1]{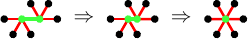}
        \caption{Contraction.}
        \label{fig:contraction}
    \end{subfigure}
    \caption{
        Movements in the extension.
        Red edges indicate bonds.
        Blue amoebots are expanding.
        Green amoebots are contracting.
        The figures show the movements in $0.5$ time steps.
    }
    \label{fig:movements}
\end{figure}


Let $S^t_i = (V, A \cup R_i)$ be the amoebot structure at $t \in [0,1]$.
Note that for all $t \in [0,1]$, $S^t_i$ is connected since $S^0_i$ is connected.
We require that $S^t_i$ is valid for all $t \in [0,1]$.
If there is a $t \in (0,1]$ such that $S^t_i$ is not valid, the amoebot structure transitions into an undefined state such that we become unable to make any statements about the structure.
We distinguish between two cases.
First, we say that a \emph{structural conflict} occurs if and only if we cannot embed $S^t_i$ in the plane, i.e., the amoebots are not able to maintain their relative positions (see \Cref{fig:conflict}).
Second, we say that a \emph{collision} occurs if and only if we can embed $S^t_i$ in the plane, but there is a forbidden intersection, i.e., two non-adjacent edges intersect or two adjacent edges intersect at a point which is not their endpoints (see \Cref{fig:collision}).

\begin{figure}[tb]
    \centering
    \includegraphics[page=1]{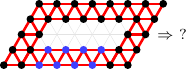}
    \caption{
        Structural conflict.
        We are not able to expand the blue amoebots horizontally without tearing up the amoebot structure.
        Hence, the expansions cause a structural conflict and the amoebot structure transitions into an undefined state.
    }
    \label{fig:conflict}
\end{figure}

\begin{figure}[tb]
    \centering
    \includegraphics[page=1]{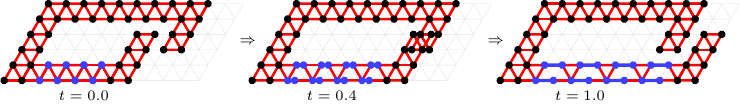}
    \caption{
        Collision.
        Initially, we have a valid amoebot structure given ($t = 0$).
        We expand the blue amoebots horizontally.
        Note that the amoebot structure is also valid at $t = 1$.
        However, for $t \in [0.25, 0.75]$, parts of the amoebot structure collide.
        Hence, the amoebot structure transitions into an undefined state.
    }
    \label{fig:collision}
\end{figure}

The detection of structural conflicts and collisions is not within the scope of this paper simply because we only consider movements where structural conflicts and collisions cannot occur.
We refer to \cite{DBLP:conf/algosensors/GuptaKMP24} for more details.

Suppose that $S^t_i$ is valid for all $t \in [0,1]$.
Note that for all $t \in (0,1)$, $S^t_i$ is no longer triangular if there is at least one expansion or contraction.
At $t = 1$, we obtain an amoebot structure $S^1_i$ which is triangular since each amoebot in $A$ has length $0$ or $1$ again.
However, $S^1_i$ is not necessarily complete.

Each arrow in the figures of \Cref{sec:meta,sec:reconfiguration,sec:motion,sec:conclusion} indicates the execution of expansions and contractions.
The left side shows $S^0_i$ (i.e., the amoebot structure after the removal of bonds in the first step), and the right side $S^1_i$ (i.e., the amoebot structure after the execution of movements in the second step).


In the third step,
the amoebots reestablish bonds until the amoebot structure becomes complete again.
Let $B'_i$ be the set of bonds afterwards.
We obtain the amoebot structure $S'_{i} = (V, A \cup B'_{i})$.
Note that $S'_{i}$ is connected, valid, triangular, and complete.


In the last step,
pairs of amoebots may perform isolated handovers as follows.
Note that each amoebot can participate in only one handover.

Consider a contracted amoebot $x$ with endpoints $x_1$ and $x_2$ and an expanded amoebot $y$ with endpoints $y_1$ and $y_2$, which are connected by a bond $b$ with endpoints $x_2$ and $y_2$ (see \Cref{fig:handover} where $x$ is marked in blue and $y$ is marked in green).
Intuitively, we want to switch the association of the expanded amoebot and the bond.
More precisely,
(i) amoebot $x$ becomes expanded (i.e., it has length $1$) and it changes its orientation to the one of the bond $b$,
(ii) amoebot $y$ becomes contracted (i.e., it has length $0$),
(iii) bond $b$ changes its orientation to the one of the previously expanded amoebot,
(iv) all bond previously adjacent to $x_2$ become adjacent to $x_1$, and
(v) all bonds previously adjacent to $y_2$ except for $b$ become adjacent to $x_2$.

Let $B_{i+1}$ be the set of bonds after the handovers and $S_{i+1} = (V, A \cup B_{i+1})$ the resulting amoebot structure.
By construction of the handovers, $S_{i+1}$ has the same properties as $S'_i$, i.e., it is connected, valid, triangular, and complete.
We call a movement \emph{valid} if and only if we are able to perform all four steps, i.e., the amoebot structure did not transition into an undefined state.

\begin{figure}[tb]
    \centering
    \includegraphics[page=1]{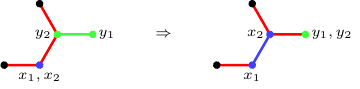}
    \caption{
        Handover.
        Red edges indicate bonds.
        The blue amoebot becomes expanded.
        The green amoebot becomes contracted.
    }
    \label{fig:handover}
\end{figure}

\begin{remark}
    We have to include the handover to ensure compatibility with algorithms for the original model and universality of the model since otherwise, it would not be possible to move through a narrow tunnel.
    However, it is not possible to define the handover as a movement in the second step without breaking previous assumptions.

    If both amoebots performed their movements simultaneously, the bond between them would change its orientation and length during the movement (see \Cref{fig:handover:simultaneous} where $x$ is marked in blue and $y$ is marked in green).
    Also, note that we cannot prevent the contracting amoebot from pulling all other incident amoebots with it such that the handover would not be an isolated local operation.
    This may result in structural conflicts or collisions.

    If the amoebots performed their movements sequentially, they would have to move faster than amoebots performing an independent movement, and they would intersect during the movement (see \Cref{fig:handover:sequential} where $x$ is marked in blue and $y$ is marked in green).

    Since all these definitions of handovers are breaking our previous assumptions, we have decided to not include handovers as a movement in the second step and include an extra step for them.
    We will refrain from using them in all of our following constructions.
\end{remark}

\begin{figure}[tb]
    \centering
    \begin{subfigure}[b]{\textwidth}
        \centering
        \includegraphics[page=1]{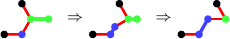}
        \caption{Simultaneous handover. The second shows a handover in $0.5$ time steps.}
        \label{fig:handover:simultaneous}
    \end{subfigure}

    \medskip

    \begin{subfigure}[b]{\textwidth}
        \centering
        \includegraphics[page=1]{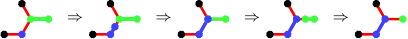}
        \caption{Sequential handover. The figure shows a handover in $0.25$ time steps.}
        \label{fig:handover:sequential}
    \end{subfigure}
    \caption{
        Possible implementations of a handover as a movement in the second step.
        Red edges indicate bonds.
        Blue amoebots are expanding.
        Green amoebots are contracting.
    }
\end{figure}


We call a sequence of valid movements a \emph{reconfiguration}.
For the sake of simplicity, we will denote each reconfiguration by the sequence of the amoebot structures between the movements, i.e., $(S_0, S_1, \dots, S_F)$ denotes a reconfiguration from $S_0$ to $S_F$.
Note that we hide the exact movements by doing so.


\begin{observation}
\label{obs:reversible}
    Each valid movement is \emph{reversible}, i.e., if there is a valid movement from $S_i$ to $S_{i+1}$, then there is also a valid movement from $S_{i+1}$ to $S_i$.
    In order to obtain the latter from the former, we proceed as follows.
    In the first step, we remove the bonds that we have added previously in the third round.
    In the second step, each amoebot performs the reverse movement.
    In the third step, we add the bonds that we have removed previously in the first round.
    We execute the fourth step before the first step where the amoebots perform the reverse handovers.



    The reversibility of valid movements immediately implies that each reconfiguration is \emph{reversible}, i.e., if there is a reconfiguration $(S_0, \dots, S_F)$, then there is also a reconfiguration $(S_F, \dots, S_0)$.
\end{observation}


In this paper, we assume that we have a centralized scheduler.
The scheduler knows the current state of the amoebot structure at all times.
At the beginning of each synchronous round, it decides for each amoebot (i) which bonds to release, (ii) which movements to perform, (iii) which bonds to reestablished and (iv) which handovers to perform.
We leave the design of distributed solutions for future work.

\subsection{Problem Statement and Our Contribution}

In this paper, we formalize and extend the joint movement extension proposed by Feldmann et al.~\cite{DBLP:journals/jcb/FeldmannPSD22}.
In the following, we provide a proof of concept.
For that, we focus on two fundamental problems of MRSs: reconfiguration and locomotion.
We study these problems from a centralized view to explore the limits of the extension.

In the \emph{reconfiguration} problem, an MRS has to reconfigure its initial structure into a given target structure.
For that, we define meta-modules of rhombical and hexagonal shape.
We show that these meta-modules are able to perform various movement primitives of other MRSs, e.g., crystalline atoms, and rectangular/hexagonal metamorphic robots.
This allows us to simulate the reconfiguration algorithms of those models.

In the \emph{locomotion} problem, an MRS has to move along an even surface as fast as possible.
We might also ask the MRS to transport an object along the way.
We present three amoebot structures that are able to move by rolling, crawling, and walking, respectively.
We analyze their velocities and compare them to other structures of similar models.



\subsection{Related Work}
\label{sec:related_work}

MRSs can be classified into various types, e.g., lattice-type, chain-type, and mobile-type \cite{DBLP:journals/jirs/AhmadzadehMA16,brunete2017current,DBLP:reference/complexity/YimWPS09,yim2002modular}.
We refer to the cited papers for examples.
We will focus on lattice-type MRSs.
These in turn can be characterized by three properties: (i) the lattice, (ii) the connectivity constraint, and (iii) the allowed  movement primitives \cite{DBLP:journals/algorithmica/AkitayaADDDFKPP21}.


Various MRSs have been defined for different lattices, e.g., \cite{DBLP:conf/icra/Chirikjian94,DBLP:conf/innovations/WoodsCGDWY13} utilize the triangular grid, and \cite{DBLP:journals/tcs/AlmethenMP23,DBLP:journals/trob/DumitrescuSY04,DBLP:journals/arobots/RusV01} utilize the Cartesian grid.
We will build meta-modules for the Cartesian and triangular grid.
Note that some MRSs were also physically realized.
Examples for MRSs using the triangular grid are hexagonal metamorphic robots \cite{DBLP:conf/icra/Chirikjian94}, HexBots \cite{DBLP:journals/jfi/SadjadiMAA12}, fractal machines \cite{DBLP:conf/icra/MurataKK94}, and catoms \cite{DBLP:conf/aaai/KirbyCAPHMG05}.
Examples for MRSs using the Cartesian grid are CHOBIE~II \cite{DBLP:conf/dars/SuzukiIKK08}, EM-Cubes \cite{DBLP:conf/icra/An08}, M-Blocks \cite{DBLP:conf/iros/RomanishinGR13}, pneumetic cellular robots \cite{DBLP:conf/icarcv/InouKK02}, and XBots \cite{DBLP:conf/iros/WhiteY07}.


We can identify four types of connectivity constraints: (i) the structure is connected at all times, (ii) the structure is connected except for moving robots, (iii) the structure is connected before and after movements, and (iv) there are no connectivity constraints.
%
Our joint movement extension falls into the first category.
Other examples are crystalline atoms \cite{DBLP:conf/icra/RusV99}, telecubes \cite{DBLP:conf/icra/SuhHY02,DBLP:conf/icra/VassilvitskiiYS02}, and prismatic cubes \cite{DBLP:conf/iros/WellerKBGG09}.
%
%
Examples for the second category are the sliding cube model \cite{DBLP:journals/ijrr/ButlerKRT04,DBLP:journals/ijrr/FitchB08}, rectangular \cite{DBLP:journals/trob/DumitrescuSY04} and hexagonal metamorphic robots \cite{DBLP:conf/icra/Chirikjian94}, and for the third category the nubot model \cite{DBLP:conf/innovations/WoodsCGDWY13} and the line pushing model \cite{DBLP:journals/tcs/AlmethenMP23}.
%
An example for the last category is the variant of the amoebot model considered by Dufoulon et al.\ \cite{DBLP:conf/podc/DufoulonKM21}.


The most common movement primitives for MRSs on the Cartesian grid are the rotation and slide primitives (see \Cref{fig:primitive:rhombical:rotation,fig:primitive:rhombical:slide}), and for MRSs on the triangular grid the rotation primitive (see \Cref{fig:primitive:hexagonal:rotation}).
Some models may constrain these movement primitives (see \Cref{fig:primitive:comparison}).
We refer to \cite{DBLP:journals/algorithmica/AkitayaADDDFKPP21,DBLP:journals/arobots/HurtadoMRA15} for a deeper discussion about possible constraints.
One way to bypass such constraints is to construct meta-modules, e.g., see \cite{DBLP:conf/iros/DeweyADGMSPC08,DBLP:journals/arobots/HurtadoMRA15,nguyen2001controlled}.
Our meta-modules implement all aforementioned movement primitives without any constraints, i.e., they perform them in place.

\begin{figure}[tb]
    \centering
    \begin{subfigure}[b]{.3\textwidth}
        \centering
        \includegraphics[page=1]{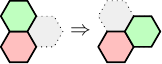}
        \caption{Rotation primitive.}
        \label{fig:primitive:hexagonal:rotation}
    \end{subfigure}
    \hfill
    \begin{subfigure}[b]{.3\textwidth}
        \centering
        \includegraphics[page=1]{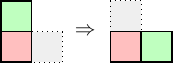}
        \caption{Rotation primitive.}
        \label{fig:primitive:rhombical:rotation}
    \end{subfigure}
    \hfill
    \begin{subfigure}[b]{.3\textwidth}
        \centering
        \includegraphics[page=1]{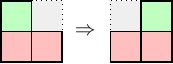}
        \caption{Slide primitive.}
        \label{fig:primitive:rhombical:slide}
    \end{subfigure}

    \smallskip

    \begin{subfigure}[b]{.3\textwidth}
        \centering
        \includegraphics[page=1]{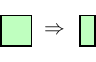}
        \caption{Contraction primitive.}
        \label{fig:primitive:rhombical:contraction}
    \end{subfigure}
    \hfill
    \begin{subfigure}[b]{.45\textwidth}
        \centering
        \includegraphics[page=1]{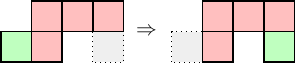}
        \caption{Tunnel primitive.}
        \label{fig:primitive:rhombical:tunnel}
    \end{subfigure}
    \hfill
    \begin{subfigure}[b]{.15\textwidth}
        \caption*{}
    \end{subfigure}

    \smallskip

    \begin{subfigure}[b]{0.45\textwidth}
        \centering
        \includegraphics[page=1]{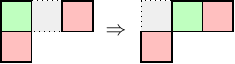}
        \caption{Leapfrog jump.}
        \label{fig:primitive:rhombical:leapfrog}
    \end{subfigure}
    \hfill
    \begin{subfigure}[b]{0.45\textwidth}
        \centering
        \includegraphics[page=1]{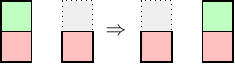}
        \caption{Straight monkey jump.}
        \label{fig:primitive:rhombical:monkey_straight}
    \end{subfigure}

    \smallskip

    \begin{subfigure}[b]{0.45\textwidth}
        \centering
        \includegraphics[page=1]{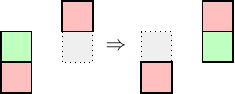}
        \caption{Diagonal monkey jumps.}
        \label{fig:primitive:rhombical:monkey_diagonal}
    \end{subfigure}
    \hfill
    \begin{subfigure}[b]{0.45\textwidth}
        \centering
        \includegraphics[page=1]{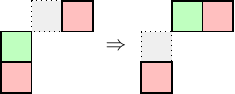}
        \caption*{}
    \end{subfigure}
    \caption{
        Examples of movement primitives.
        The green modules are moving, respectively.
        Models that utilize the Cartesian grid usually assume square modules.
        In \Cref{sec:meta:rhombical}, we utilize rhombical modules instead.
    }
    \label{fig:primitive}
\end{figure}

\begin{figure}[tb]
    \centering
    \begin{subfigure}[b]{.45\textwidth}
        \centering
        \includegraphics[page=1]{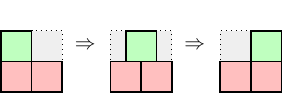}
        \caption{Sliding model, e.g., \cite{DBLP:journals/arobots/HurtadoMRA15}.}
    \end{subfigure}
    \hfill
    \begin{subfigure}[b]{.45\textwidth}
        \centering
        \includegraphics[page=1]{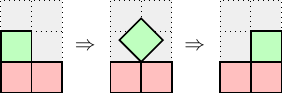}
        \caption{Pivoting model, e.g., \cite{DBLP:journals/algorithmica/AkitayaADDDFKPP21}.}
    \end{subfigure}
    \caption{
        Constrained movement primitives.
        The green modules are moving.
        The gray cells must be empty.
        Both primitives have the same result while the pivoting model requires more free space than the sliding model.
    }
    \label{fig:primitive:comparison}
\end{figure}



From all aforementioned models, crystalline atoms \cite{DBLP:conf/icra/RusV99}, telecubes \cite{DBLP:conf/icra/SuhHY02,DBLP:conf/icra/VassilvitskiiYS02}, and prismatic cubes \cite{DBLP:conf/iros/WellerKBGG09} are the closest to our joint movement extension.
In these MRSs, each robot has the shape of a unit square and is able to expand an arm from each side by half a unit.
Adjacent robots are able to attach and detach their arms.
Similar to our joint movement extension, a robot can move attached robots by expanding or contracting its arms.
In contrast to amoebots, a line of robots cannot reconfigure to any other shape since each pair of robots can only have a single point of contact \cite{DBLP:journals/comgeo/AloupisCDDFLORAW09}.
In order to allow arbitrary reconfigurations, the robots are combined into meta-modules of square shape
that are capable of various movement primitives, e.g., the rotation, slide, contraction, and tunnel primitives (see \Cref{fig:primitive:rhombical:rotation,fig:primitive:rhombical:contraction,fig:primitive:rhombical:slide,fig:primitive:rhombical:tunnel}).
%
Our rhombical meta-modules can implement all these movement primitives.
The rhombical shape has two advantages compared to the square shape.
First, we can implement further movement primitives that provide us with a simple way to construct a walking structure.
Second, we can utilize rhombical meta-modules as a basis for hexagonal meta-modules.

\section{Meta-Modules}
\label{sec:meta}

In this section, we will combine multiple amoebots to meta-modules.
In other models for programmable matter and modular robots, meta-modules have proven to be very useful.
For example, they allow us to bypass restrictions on the reconfigurability \cite{DBLP:conf/iros/DeweyADGMSPC08,DBLP:conf/icra/VassilvitskiiKRSY02} and to simulate (reconfiguration) algorithms for other models \cite{DBLP:journals/comgeo/AloupisBDDFIW13,DBLP:conf/icra/ParadaSS16}.

In the first two subsections, we will present meta-modules of rhombical and hexagonal shape, respectively.
In the third subsection, we show how to construct such meta-modules from an arbitrary amoebot structure.
In the fourth subsection, we discuss how to increase the stability of the structure during the movements.


\subsection{Rhombical Meta-Modules}
\label{sec:meta:rhombical}


Let $\ell$ be a positive even integer.
Our \emph{rhombical meta-module} consists of $\ell^2/2$ uniformly oriented expanded amoebots that we arrange into a rhombus of side length $\ell - 1$ (see the left side of \Cref{fig:meta:rhombical:contraction}).
We obtain a parallelogram of side lengths $\ell - 1$ and $\ell/2 - 1$ if we \emph{contract} all amoebots (see \Cref{fig:meta:rhombical:contraction}).
Note that we have to remove some bonds to perform the contraction.
We can \emph{expand} the parallelogram again by reversing the contractions.

\begin{figure}[tb]
    \centering
    \includegraphics[page=1]{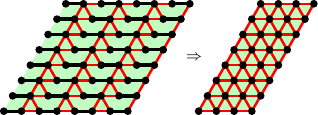}
    \caption{Contraction primitive.}
    \label{fig:meta:rhombical:contraction}
\end{figure}

\begin{lemma}
\label{lem:meta:rhombical:contraction}
    Our implementation of the contraction and expansion primitive requires a single round, respectively.
\end{lemma}


There are exactly two possibilities to arrange the uniformly oriented expanded amoebots in a rhombus.
We can \emph{reorientate} all amoebots within a rhombus as follows (see \Cref{fig:meta:rhombical:reorientation}).
First, we partition the meta-module into $\frac{\ell}{2} \times \frac{\ell}{2}$ submodules.
Note that each submodule consists of exactly $2$ amoebots.
Then, we color the submodules in a checkerboard pattern.
Finally, we reorientate both colors independently of each other within two phases.
Each phase consists of two rounds.
In the first round, we contract each submodule of one color into its shorter diagonal.
The submodules of the other color ensure that the meta-module remains connected.
In the second round, we expand all contracted amoebots into the other corner of its submodule.

\begin{figure}[tb]
    \centering
    \includegraphics[page=1]{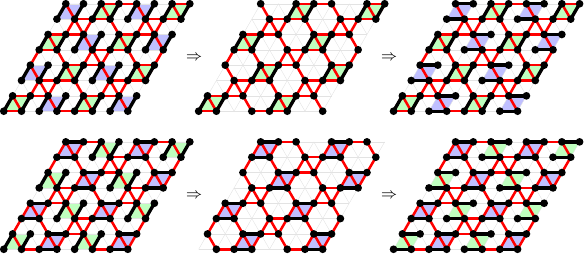}
    \caption{Reorientation primitive.}
    \label{fig:meta:rhombical:reorientation}
\end{figure}

\begin{lemma}
\label{lem:meta:rhombical:reorientation}
    Our implementation of the reorientation primitive requires $4$ rounds.
\end{lemma}

Furthermore, there are three possibilities to align the sides of a rhombus to the axes of the triangular grid.
By sliding each second row along its axis to the other side, we can \emph{realign} the other axis of the rhombus to the third axis of the triangular grid (see \Cref{fig:meta:rhombical:realignment}).
Note that in combination with the reorientation primitive, we are able to align a rhombus with any two axes of the triangular grid.

\begin{figure}[tb]
    \centering
    \includegraphics[page=1]{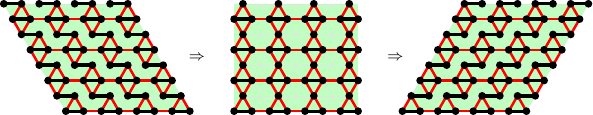}
    \caption{Realignment primitive.}
    \label{fig:meta:rhombical:realignment}
\end{figure}

\begin{lemma}
\label{lem:meta:rhombical:realignment}
    Our implementation of the realignment primitive requires $2$ rounds.
\end{lemma}


We can arrange the meta-modules on a rhombical tesselation of the plane if they are all aligned to the same axes (see \Cref{fig:meta:rhombical:tesselation}).
Note that due to the triangular grid, the meta-modules are not connected diagonally everywhere.
Hence, we will only consider meta-modules connected if their sides are connected.

\begin{figure}[tb]
    \centering
    \includegraphics[page=1]{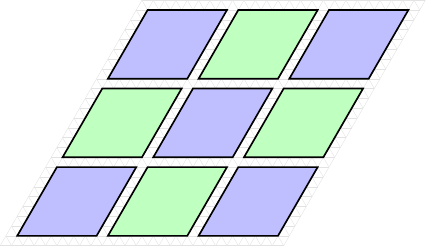}
    \caption{
        Rhombical tesselation.
        Note that there are spaces between the meta-modules since a node cannot be occupied by more than one amoebot.
        However, the spaces do not contain any nodes.
    }
    \label{fig:meta:rhombical:tesselation}
\end{figure}

In the following, we introduce three movement primitives: the slide, rotation, and $k$-tunnel primitive.
Our implementations of these primitives are similar to the ones for crystalline robots (e.g., \cite{DBLP:journals/comgeo/AloupisCDDFLORAW09}) and teletubes (e.g., \cite{DBLP:conf/icra/VassilvitskiiYS02}).


In the \emph{slide primitive},
we move a meta-module $R_1$ along two adjacent substrate meta-modules $R_2$ and $R_3$ (see \Cref{fig:primitive:rhombical:slide}).
%
We realize the primitive as follows (see \Cref{fig:meta:rhombical:slide:fast}).
We assume that all amoebots are orientated into the movement direction.
Otherwise, we apply the reorientation primitive.
With respect to \Cref{fig:meta:rhombical:slide:fast}, let $L_1$ denote the uppermost layer of $R_2$ and $R_3$, and $L_2$ the second uppermost layer of $R_2$ and $R_3$.
Our slide primitive consists of two rounds.
In the first round,
we contract all amoebots in $L_1$ after removing all bonds between $L_1$ and $R_1$ except the last one in the movement direction, and all bonds between $L_1$ and $L_2$ except the last one in the opposite direction.
This moves $R_1$ into its target position.
In the second round,
we restore $R_2$ and $R_3$.
For that, we expand $L_1$ again after removing all bonds between $L_1$ and $R_1$, and between $L_1$ and $L_2$ except the last ones in the movement direction, respectively.
This ensures that $R_1$ stays in place.

\begin{figure}[tb]
    \centering
    \includegraphics[page=1]{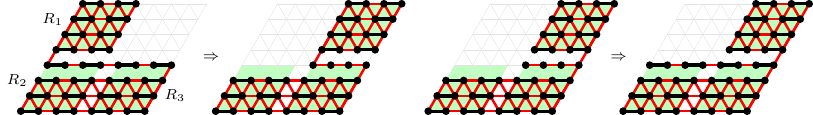}
    \caption{Slide primitive.}
    \label{fig:meta:rhombical:slide:fast}
\end{figure}

\begin{lemma}
\label{lem:meta:rhombical:slide:fast}
    Our implementation of the slide primitive requires $2$ rounds.
\end{lemma}


In the \emph{rotation primitive}, we move a meta-module $R_1$ around another meta-module $R_2$ (see \Cref{fig:primitive:rhombical:rotation}).
We realize the primitive as follows (see \Cref{fig:meta:rhombical:rotation}).
We assume that all amoebots are orientated towards each other.
Otherwise, we apply the reorientation primitive.
Our rotation primitive consists of two rounds.
In the first round, we \emph{pull} $R_1$ into $R_2$ by contracting both meta-modules into parallelograms.
The resulting contracted meta-modules form a rhombus consisting of $\ell^2$ contracted amoebots.
Since all amoebots are contracted, we can rotate $R_1$ and $R_2$ within that rhombus by exchanging amoebots.
This allows us in the second round to \emph{push} $R_1$ out of $R_2$ into $R_1$'s target position.

\begin{figure}[tb]
    \centering
    \includegraphics[page=1]{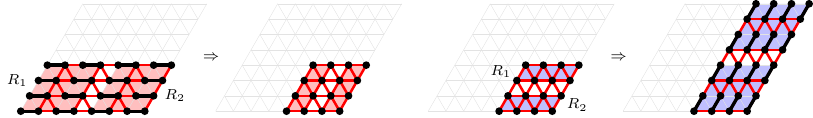}
    \caption{Rotation primitive for rhombical meta-modules. Red meta-modules perform a pull operation, and blue meta-modules a push operation.}
    \label{fig:meta:rhombical:rotation}
\end{figure}

\begin{lemma}
\label{lem:meta:rhombical:rotation}
    Our implementation of the rotation primitive requires $2$ rounds.
\end{lemma}


In the \emph{$k$-tunnel primitive}, we move a meta-module $R$ through a simple path of meta-modules with $k$ corners to the other end (see \Cref{fig:primitive:rhombical:tunnel}).
We realize the primitive as follows (see \Cref{fig:meta:rhombical:tunnel:fast}).
We assume that all meta-modules in the corners are orientated towards the preceding meta-module, and all other meta-modules are oriented along the path.
Otherwise, we apply the reorientation primitive.
We will not move $R$ directly through the path.
Instead, we will make use of pull and push operations (see \Cref{fig:meta:rhombical:rotation}).

Consider a line of at least 4 meta-modules with two contracted meta-modules at one end.
By expanding those two meta-modules and contracting the two meta-modules at the other end, we transfer a meta-module from one end to the other end without changing the length of the line (see the third round in \Cref{fig:meta:rhombical:tunnel:fast}).
If we have only a line of 3 meta-modules, it suffices to contract and expand one meta-module, respectively (see the second round in \Cref{fig:meta:rhombical:tunnel:fast}).
Note that we have to remove most of the bonds along the line to permit the line to move freely.
The pull and push operations allow us to `transfer' $R$ from one corner to the next corner in a single round.
Note that the rotation primitive is a special case of the $k$-tunnel primitive.

\begin{figure}[tb]
    \centering
    \includegraphics[page=1]{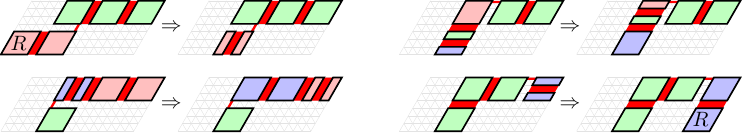}
    \caption{$k$-tunnel primitive. Red meta-modules perform a pull operation, and blue meta-modules a push operation.}
    \label{fig:meta:rhombical:tunnel:fast}
\end{figure}

\begin{lemma}
\label{lem:meta:rhombical:tunnel:fast}
    Our implementation of the $k$-tunnel primitive requires $k + 1$ rounds.
\end{lemma}




Akitaya et al.\ \cite{DBLP:journals/algorithmica/AkitayaADDDFKPP21} have introduced two further movement primitives for rectangular MRSs: the \emph{leapfrog} and the \emph{monkey} primitive (see \Cref{fig:primitive:rhombical:leapfrog,fig:primitive:rhombical:monkey_straight,fig:primitive:rhombical:monkey_diagonal}).
In these primitives, the moving module performs a `jump' between two not adjacent modules.
We are able to transfer these primitives to our rhombical meta-modules.
The idea is to divide the meta-modules into submodules and then build a bridge between the start and target position of the moving meta-module by applying slide and $k$-tunnel primitives (see \Cref{fig:meta:pivot}).

\begin{figure}[tb]
    \begin{subfigure}[b]{\textwidth}
        \centering
        \includegraphics[page=1]{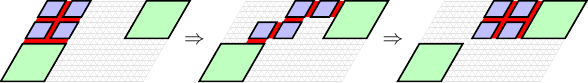}
        \caption{Leapfrog jump. \smallskip}
    \end{subfigure}
    \begin{subfigure}[b]{\textwidth}
        \centering
        \includegraphics[page=1]{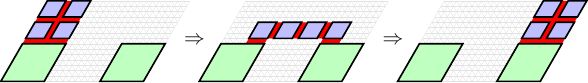}
        \caption{Straight monkey jump. \smallskip}
    \end{subfigure}
    \begin{subfigure}[b]{\textwidth}
        \centering
        \includegraphics[page=1]{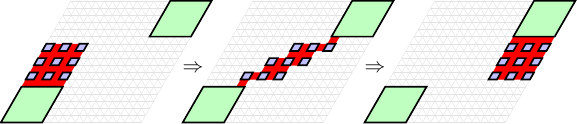}

        \medskip

        \includegraphics[page=1]{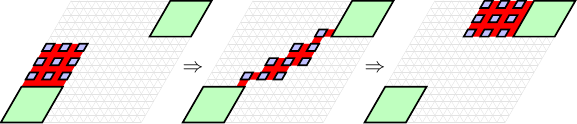}
        \caption{Diagonal monkey jumps.}
    \end{subfigure}
    \caption{
        Leapfrog and monkey jump primitives.
        Each step consists of multiple rounds.
    }
    \label{fig:meta:pivot}
\end{figure}

\begin{lemma}
    Our implementation of the leapfrog and monkey primitives requires $O(1)$ rounds, respectively.
\end{lemma}



\subsection{Hexagonal Meta-Modules}
\label{sec:meta:hexagonal}


Let $\ell$ be an even integer as before.
Our \emph{hexagonal meta-module} consists of three rhombical meta-modules of side length $\ell - 1$ (see \Cref{fig:meta:hexagonal:rearrangement}).
We arrange them into a hexagon of alternating side lengths $\ell$ and $\ell - 1$.


There are two possibilities to arrange the rhombical meta-modules in a hexagonal meta-module (see \Cref{fig:meta:hexagonal:rearrangement}).
We can \emph{switch} between them as follows.
We assume that all amoebots are orientated as shown in \Cref{fig:meta:hexagonal:rearrangement}.
Otherwise, we apply the reorientation primitive.
Our implementation requires two rounds.
However, for the sake of comprehensibility, we will split both rounds into two rounds, respectively.

The idea is to apply a similar technique as in the realignment primitive for rhombical meta-modules.
In the first round, each rhombical meta-module contracts each second row without shifting the other rows.
Each rhombical meta-module within the hexagonal meta-module can be split into an equilateral triangle of side length $\ell - 1$, and an equilateral triangle of side length $\ell - 2$.
In the second and third round, the rhombical meta-modules exchange the smaller triangles.
For that, we first contract all expanded amoebots in each smaller triangle such that the amoebots form lines between the adjacent bigger triangles.
Note that the amoebots in the corners of the hexagonal meta-module ensure connectivity to adjacent hexagonal meta-modules.
Then, we expand the other amoebots in each smaller triangle into the nodes previously occupied by the expanded amoebots.
In the fourth round, each resulting rhombical meta-module expands each second row into the opposite direction (of the contracted amoebots of its bigger triangle) -- again without shifting the other rows.
Finally, note that we can combine the first and second round resp.\ the third and fourth round.
For that, we have to remove all necessary bonds at once, respectively.

\begin{figure}[ptb]
    \centering
    \includegraphics[page=1]{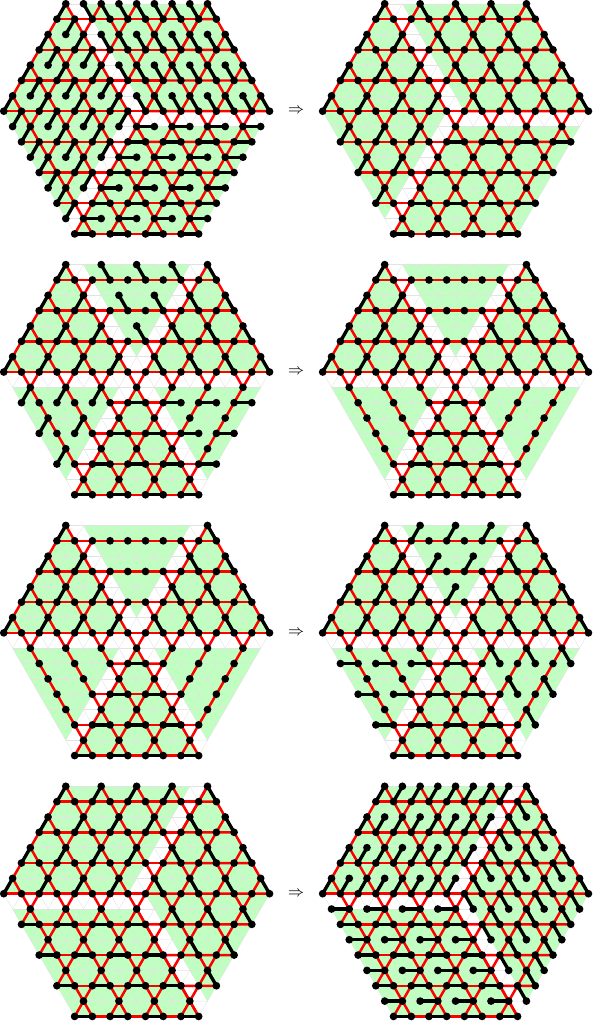}
    \caption{Switching primitive.}
    \label{fig:meta:hexagonal:rearrangement}
\end{figure}

\begin{lemma}
\label{lem:meta:hexagonal:rearrangement}
    Our implementation of the switching primitive requires $2$ rounds.
\end{lemma}


We can arrange the meta-modules on a hexagonal tesselation of the plane (see \Cref{fig:meta:hexagonal:tesselation}).
In the following, we introduce the rotation primitive for hexagonal meta-modules.

\begin{figure}[tb]
    \centering
    \includegraphics[page=1]{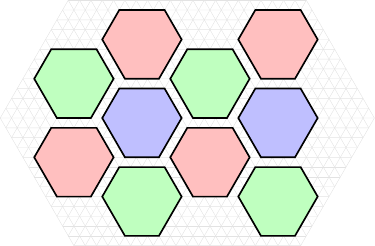}
    \caption{
        Hexagonal tesselation.
        Note that there are spaces between the meta-modules since a node cannot be occupied by more than one amoebot.
        However, the spaces do not contain any nodes.
    }
    \label{fig:meta:hexagonal:tesselation}
\end{figure}


In the \emph{rotation primitive},
we move a hexagonal meta-module $H_1$ around another hexagonal meta-module $H_2$ as follows (see \Cref{fig:primitive:hexagonal:rotation,fig:meta:hexagonal:rotate:simplified}).
We arrange $H_2$ such that a rhombical meta-module $R_2$ is adjacent to both the old and new position of $H_1$,
and $H_1$ such that the rhombical meta-module $R_1$ adjacent to $R_2$ is aligned to the same axes as $R_2$.
We contract $R_1$ and $R_2$, and then expand them into the direction of the new position of $H_1$ (compare to \Cref{lem:meta:rhombical:rotation}).
This movement primitive requires two rounds.
Note that additional steps may be necessary to switch or reorientate the rhombical meta-modules beforehand.

\begin{figure}[tb]
    \centering
    \includegraphics[page=1]{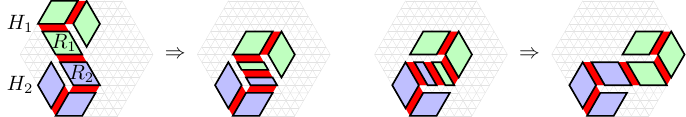}
    \caption{
    Rotation primitive for hexagonal meta-modules.
    }
    \label{fig:meta:hexagonal:rotate:simplified}
\end{figure}

\begin{lemma}
\label{lem:meta:hexagonal:rotate:simplified}
    Our implementation of the rotation primitive requires $2$ rounds.
\end{lemma}



\subsection{Meta-Module Construction}
\label{sec:meta:construction}


So far, we have considered amoebot structures that are already composed of rhombical and hexagonal meta-modules.
However, the initial amoebot structure may be arbitrary.
In this section, we consider the transformation of a line of amoebots to a line of rhombical and hexagonal meta-modules, respectively.
The formation of a line of amoebots from an arbitrary amoebot structure has been considered in previous work (e.g., \cite{DBLP:journals/dc/LunaFSVY20}).


First, we form a line of $r$ rhombical meta-modules of side length $\ell-1$ from a line of $n = r \cdot \ell^2/2$ amoebots as follows (see \Cref{fig:meta:rhombical:construction}).
First, we divide the line of amoebots into groups of $\ell$ amoebots.
The idea is to apply the realignment primitive (see \Cref{lem:meta:rhombical:realignment}) to rotate each second group once counterclockwise by $60^\circ$ and each other group twice clockwise by $60^\circ$ without breaking the connectivity along the line.
However, in order to apply the realignment primitive, we first have to expand the groups into parallelograms of side lengths $\ell-1$ and $1$.

Therefore, the procedure consists of the following seven phases.
In the first six phases, we alternate between two types of phases.
In the first type, we possibly remove the parallelograms of the previous phase (third, and fifth phase) and form the parallelograms for the next phase (first, third, and fifth phase).
In the second type, we apply the realignment primitive to rotate the parallelograms by $60^\circ$ (second, fourth, and sixth phase).
Finally, in the seventh phase, we obtain the rhombical meta-modules by expanding all contracted amoebots.
We obtain the following theorem.


\begin{figure}[tb]
    \centering
    \includegraphics[page=1]{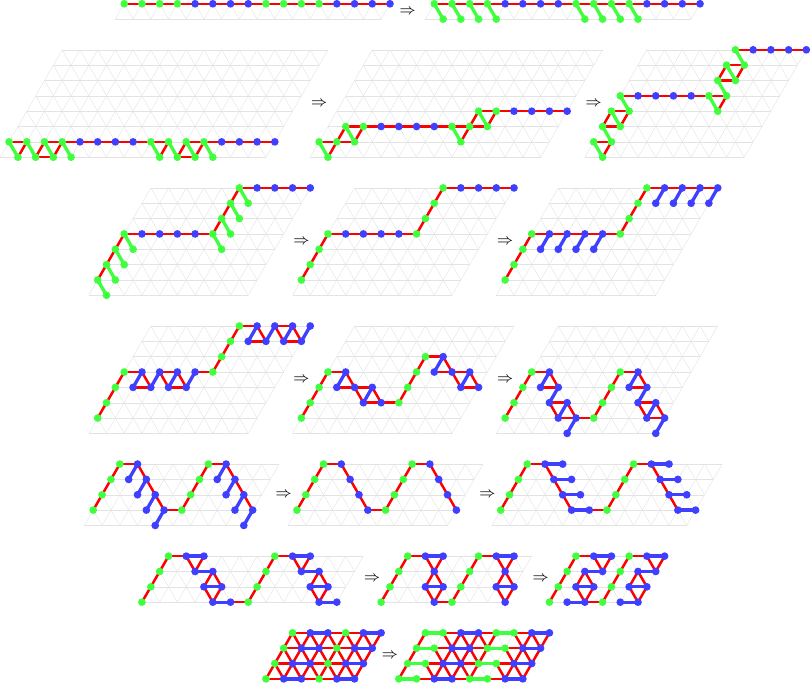}
    \caption{
        Construction of rhombical meta-modules.
        Adjacent amoebots of the same color belong to the same group.
        Each row shows a phase.
    }
    \label{fig:meta:rhombical:construction}
\end{figure}

\begin{theorem}
\label{th:meta:rhombical:construction}
    There is a centralized algorithm that forms a line of $r$ rhombical meta-modules of side length $\ell-1$ from a line of $n = r \cdot \ell^2/2$ amoebots within $12$ rounds.
\end{theorem}

\begin{proof}
    Each removal and formation requires one round.
    Hence, the first phase requires a single round, and the third and fifth phase require two rounds, respectively.
    By \Cref{lem:meta:rhombical:realignment}, each phase of the second type requires two rounds.
    The seventh phase requires a single round since we only expand all contracted amoebots.
    Therefore, the algorithm requires $12$ rounds overall.
\end{proof}

\begin{remark}
    \Cref{th:meta:rhombical:construction} also improves a result by Feldmann et al.\ \cite{DBLP:journals/jcb/FeldmannPSD22}.
    With the help of handovers, they showed how to reconfigure a line of amoebots into a rhombus within $O(\log n)$ rounds.
    For $r = 1$, we obtain the same result within $12$ rounds without using handovers.
\end{remark}


Second, we form a line of $h = \frac{r}{3}$ hexagonal meta-modules of alternating side lengths $\ell$ and $\ell - 1$ from a line of $r$ rhombical meta-modules of side length $\ell - 1$ as follows (see \Cref{fig:meta:hexagonal:construction}).
We assume that all rhombical meta-modules are oriented along the line.
Otherwise, we apply the reorientation primitive.
We first divide the line of rhombical meta-modules into groups of $3$ rhombical meta-modules and apply the following three steps to each group.

In the first step, we apply the rotation primitive to rotate the first meta-module around the second one.
In the second step, we apply the reorientation primitive to the first and third meta-module such that in the third step, we can apply the realignment primitive to the them to form a hexagonal meta-module.
We obtain the following theorem.

%

\begin{figure}[tb]
    \centering
    \includegraphics[page=1]{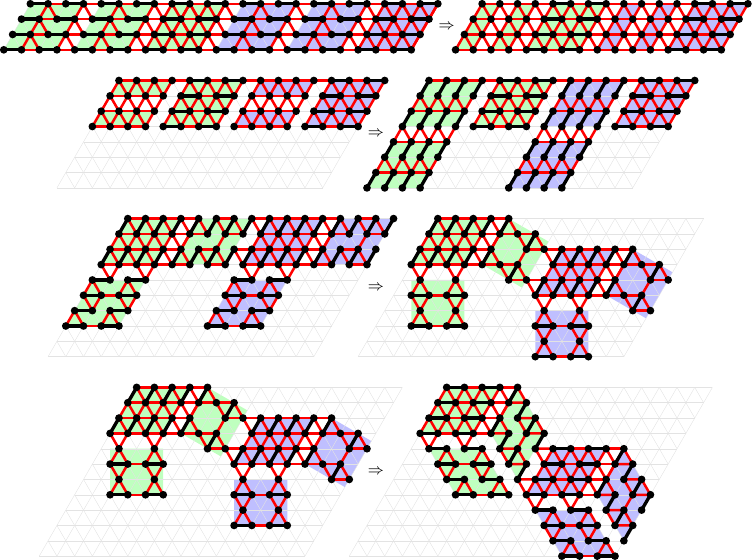}
    \caption{
        Construction of hexagonal meta-modules.
        Adjacent meta-modules of the same color belong to the same group.
        The first two rounds show the application of the rotation primitive.
        The last two rounds show the application of the realignment primitive.
        We have omitted the application of the reorientation primitive.
    }
    \label{fig:meta:hexagonal:construction}
\end{figure}

\begin{theorem}
\label{th:meta:hexagonal:construction}
    There is a centralized algorithm that forms a line of $h = \frac{r}{3}$ hexagonal meta-modules of alternating side lengths $\ell$ and $\ell - 1$ from a line of $r$ rhombical meta-modules of side length $\ell - 1$ within $8$ rounds.
\end{theorem}

\begin{proof}
    By \Cref{lem:meta:rhombical:rotation}, the first step requires $2$ rounds.
    By \Cref{lem:meta:rhombical:reorientation}, the second step requires $4$ rounds.
    By \Cref{lem:meta:rhombical:realignment}, the third step requires $2$ rounds.
    Therefore, the algorithm requires $8$ rounds overall.
\end{proof}

\begin{corollary}
\label{cor:meta:hexagonal:construction}
    There is a centralized algorithm that forms a line of $h$ hexagonal meta-modules of alternating side lengths $\ell$ and $\ell - 1$ from a line of $n = 3 \cdot h \cdot \ell^2/2$ amoebots within $20$ rounds.
\end{corollary}

\begin{proof}
    The corollary directly follows from \Cref{th:formation:rhombical,th:formation:hexagonal}.
    Note that all rhombical meta-modules are oriented along the line after applying \Cref{th:formation:rhombical} such that no additional reorientations are necessary.
\end{proof}



\subsection{Stability}
\label{app:meta:stability}

One might argue that removing too many bonds may destabilize the structure during the movements.
It is therefore desirable to maintain the connectivity between the amoebots as much as possible.
In the following, we propose two measures for stability.


First, we consider the connectivity of the amoebot structure during each movement.
Note that we only decrease the connectivity in the first step of each movement where we may remove bonds.
In the second and fourth step, the connectivity does not change, and in the third step, we can only increase the connectivity by adding new bonds.
Ideally, we want to just remove a constant fraction of all incident bonds of a connected component.

Formally, let $\Gamma_{G}(U) = \{ \{u,v\} \in E \mid u \in U \land v \not\in U \}$ where $G = (V, E)$.
Recall that $S_i$ denotes the amoebot structure before the first step, and $S_i^0$ the amoebot structure $S_i$ after the first step.

\begin{definition}
\label{def:stable:locally}
    We call a reconfiguration $(S_0, \dots, S_F)$ \emph{locally stable} if and only if
    \begin{equation*}
        \forall i \in \{ 0, \dots, F-1 \}: |\Gamma_{S_i^0}(U)|=\Theta(|\Gamma_{S_i}(U)|) \text{ for all connected components } U \subseteq S_i.
    \end{equation*}
\end{definition}

\begin{remark}
    There are locally stable reconfigurations $(S_0, \dots, S_F)$ such that the reverse reconfiguration $(S_F, \dots, S_0)$ is not locally stable.
    For example, consider the second step of the second last phase of our algorithm that forms a line of $r$ rhombical meta-modules of side length $\ell-1$ from a line of $n = r \cdot \ell^2/2$ amoebots (see \Cref{th:meta:rhombical:construction,fig:meta:rhombical:construction}).
    For $r = 1$, we have $\ell = O(\sqrt n)$.
    Recall that each group consists of $\ell$ amoebots.
    The reconfiguration only removes a constant fraction of the bonds between two groups and adds $O(\sqrt n)$ bonds afterwards.
    Hence, the reverse reconfiguration has to remove these $O(\sqrt n)$ bonds except for one.
\end{remark}


Second, we consider the connectivity of the amoebot structure during the whole reconfiguration.
Ideally, a stable reconfiguration algorithm should keep the connectivity as high as possible during the whole process.
The weakest point of a reconfiguration $(S_0, \dots, S_F)$ is the amoebot structure $S_i$ with the lowest connectivity.
However, the achievable connectivity highly depends on the initial amoebot structure $S_0$ and the target amoebot structure $S_F$.
Hence, the connectivity of $S_i$ should be at least as high as the one of $S_0$ and $S_F$ with the lower connectivity.
Recall that $\Gamma_{G}(U) = \{ \{u,v\} \in E \mid u \in U \land v \not\in U \}$ where $G = (V, E)$.
Let
\begin{equation*}
    \alpha(G) = \min_{U \subseteq V, 1 \leq |U| \leq |V|/2} \frac{|\Gamma(U)|}{|U|}.
\end{equation*}

\begin{definition}
\label{def:stable:globally}
    We call a reconfiguration $(S_0, \dots, S_F)$ \emph{globally stable} if and only if
    \begin{equation*}
        \min_{ i \in \{ 0, \dots, F \} } \alpha(G_{S_i})=\Omega(\min\{\alpha(G_{S_0}),\alpha(G_{S_F})\}).
    \end{equation*}
\end{definition}

\begin{definition}
\label{def:stable}
    We call a reconfiguration $(S_0, \dots, S_F)$ \emph{stable} if and only if it is locally and globally stable.
\end{definition}


Note that both constructions in \Cref{sec:meta:construction} are stable.
In order to obtain stable reconfigurations in general, locally stable movement primitives can be helpful.
All so far presented movement primitives are locally stable except for the slide, rotation and $k$-tunnel primitives for rhombical meta-modules, and the switching primitive for hexagonal meta-modules.
In the following, we present alternative implementations that increase the connectivity for these primitives.

First, consider the slide primitive.
Recall that we need to slide a meta-module $R_1$ along two meta-modules $R_2$ and $R_3$ (see \Cref{fig:primitive:rhombical:slide}).
We refer to \Cref{sec:meta:rhombical} for details.
The idea is to enlarge the moving part as follows (see \Cref{fig:meta:rhombical:slide:slow}).
For the sake of simplicity, we assume that $\ell$ is divisible by $8$.
We partition $R_2$ and $R_3$ into $4 \times 4$ submodules, respectively.
The slide primitive consists of the following four rounds.
In the first round, we utilize the submodules adjacent to the start and target position of $R_1$ to move $R_1$ to the center position.
In the second round, we restore $R_2$ and $R_3$ again.
In the last two rounds, we mirror the first two rounds to move $R_1$ to its target position.

\begin{figure}[tb]
    \centering
    \includegraphics[page=1]{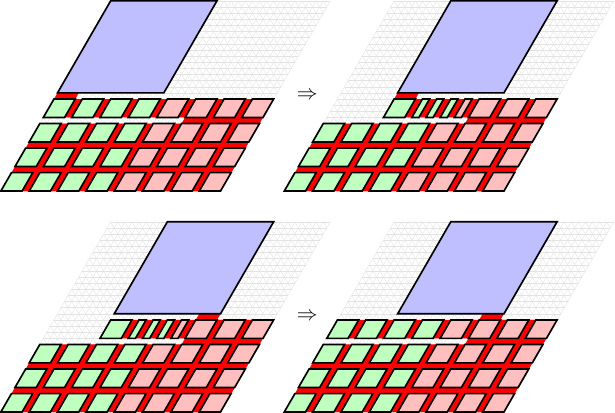}
    \caption{
        Slide primitive with a higher expansion.
        The rhombi indicate the meta-modules.
        The blue rhombi indicate $R_1$, and the green and red rhombi indicate $R_2$ and $R_3$, respectively.
        The red color between the meta-modules indicate bonds.
    }
    \label{fig:meta:rhombical:slide:slow}
\end{figure}

\begin{lemma}
\label{lem:meta:rhombical:slide:slow}
    Our alternative implementation of the slide primitive requires $4$ rounds.
\end{lemma}


Second, consider the rotation and $k$-tunnel primitives.
Since the rotation primitive is a special case of the $k$-tunnel primitive, we will constrain the explanation to the latter.
Recall that we need to move a meta-module $R$ through a simple path of meta-modules with $k$ corners to the other end (see \Cref{fig:primitive:rhombical:tunnel}).
We refer to \Cref{sec:meta:rhombical} for details.

We now discuss how to perform a single step of the $k$-tunnel, i.e., how to transfer a meta-module between two corners.
Consider the line of meta-modules between those corners.
One end contains two contracted meta-modules.
We distinguish between to cases: (i) there is at least one other meta-module between the corners (see \Cref{fig:meta:rhombical:tunnel:slow:2}), and (ii) there is no other meta-module between the corners (see \Cref{fig:meta:rhombical:tunnel:slow:1}).

\begin{figure}[ptb]
    \centering
    \includegraphics[page=1]{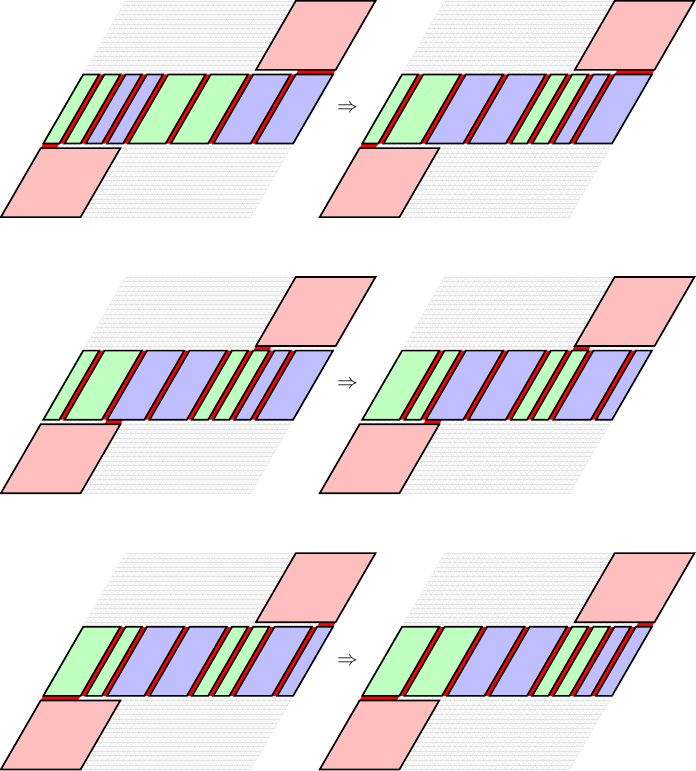}
    \caption{
        First case for the $k$-tunnel primitive.
        The rhombi and parallelograms indicate the meta-modules and submodules, respectively.
        Adjacent parallelograms of the same color belong to the same meta-module.
        The green and blue parallelograms indicate the considered line of meta-modules.
        The red color between the meta-modules indicate bonds.
    }
    \label{fig:meta:rhombical:tunnel:slow:2}
\end{figure}

\begin{figure}[ptb]
    \centering
    \includegraphics[page=1]{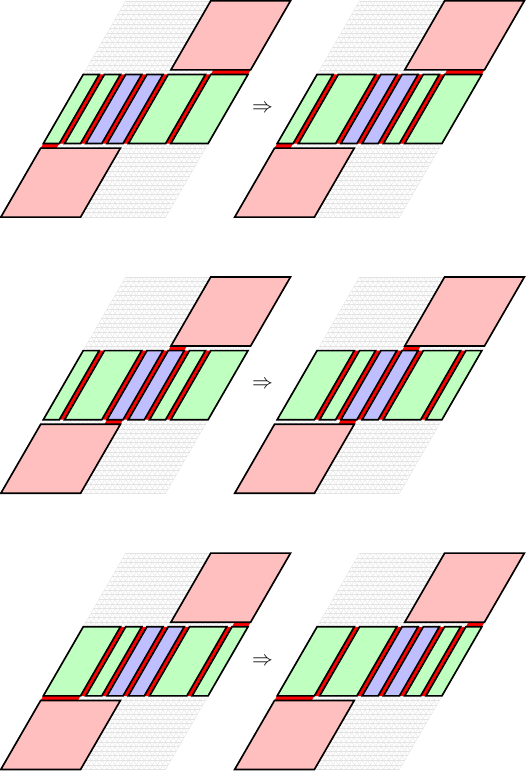}
    \caption{
        Second case for the $k$-tunnel primitive.
        The rhombi and parallelograms indicate the meta-modules and submodules, respectively.
        Adjacent parallelograms of the same color belong to the same meta-module.
        The green and blue parallelograms indicate the considered line of meta-modules.
        The red color between the meta-modules indicate bonds.
    }
    \label{fig:meta:rhombical:tunnel:slow:1}
\end{figure}

For the sake of simplicity, we assume that $\ell$ is divisible by $4$.
We divide each meta-module of the line into two submodules orthogonally to the movement direction.
Each submodule forms a parallelograms of side lengths $\ell - 1$ and $\ell/2 - 1$.
We assume that all expanded amoebots are orientated into the movement direction.
This allows us to contract the submodules into parallelograms of side lengths $\ell - 1$ and $\ell/4 - 1$.

The difficulty lies in maintaining enough bonds to the meta-modules attached to the corners.
We will ensure that there is always a submodule that maintains its bonds.
A single step of the $k$-tunnel primitive consists of the following three rounds.

In the first round, we contract or expand all submodules that we would have in the original version of the $k$-tunnel primitive -- except for the ones at the two ends of the line.
We use those two to maintain the bonds to attached meta-modules.

In the second round, we swap the states of the submodules of the meta-module at both ends.
We use the third submodule of both ends to maintain the bonds to attached meta-modules.
In the third round, we complete the transfer by contracting and expanding the remaining two submodules.
We use the submodules at the two ends of the line again to maintain the bonds to attached meta-modules.

\begin{lemma}
\label{lem:meta:rhombical:tunnel:slow}
    Our alternative implementation of the $k$-tunnel primitive requires $3 \cdot (k + 1)$ rounds.
\end{lemma}

\begin{corollary}
    Our alternative implementation of the rotation primitive requires $6$ rounds.
\end{corollary}

\begin{remark}
    We can utilize our alternative implementation of the rotation primitive to make the algorithms that form a line of hexagonal meta-modules locally stable (see \Cref{th:meta:hexagonal:construction,cor:meta:hexagonal:construction}).
    Note that the first step of the rotation primitive, i.e., the contraction of the meta-modules, is already locally stable (see first round in \Cref{fig:meta:hexagonal:construction}).
    Hence, it suffices to only use the alternative implementation on the second step, i.e., the expansion  of the meta-modules (see second round in \Cref{fig:meta:hexagonal:construction}).
    Therefore, the required number of rounds of \Cref{th:meta:hexagonal:construction,cor:meta:hexagonal:construction} increases by $2$ to $10$ and $22$, respectively.
\end{remark}


Finally, consider the switching primitive.
Recall that we need to switch between the two possibilities to arrange the rhombical meta-modules in a hexagonal meta-module (see \Cref{fig:meta:hexagonal:rearrangement}).
We refer to \Cref{sec:meta:hexagonal} for details.

Consider our implementation in \Cref{fig:meta:hexagonal:rearrangement}.
Note that only the second round does not satisfy our connectivity requirement.
Hence, we alter the way we exchange the smaller triangles between the rhombical meta-modules as follows.
We replace the second and third round in \Cref{fig:meta:hexagonal:rearrangement} with the four rounds in \Cref{fig:meta:hexagonal:rearrangement:alt}.
The idea is to apply a similar technique as in the reorientation primitive for rhombical meta-modules.
Instead of exchanging all amoebots at once, we exchange them in two phases.
In the first phase, we exchange each second pair of amoebots on each line between the bigger triangles (see first and second round in \Cref{fig:meta:hexagonal:rearrangement:alt}).
In the second phase, we exchange all other pairs of amoebots (see third and fourth round in \Cref{fig:meta:hexagonal:rearrangement:alt}).
Note that we can again combine the first two (first round in \Cref{fig:meta:hexagonal:rearrangement} and first round in \Cref{fig:meta:hexagonal:rearrangement:alt}) and last two rounds (last round in \Cref{fig:meta:hexagonal:rearrangement} and last round in \Cref{fig:meta:hexagonal:rearrangement:alt}), respectively.

\begin{figure}[ptb]
    \centering
    \includegraphics[page=1]{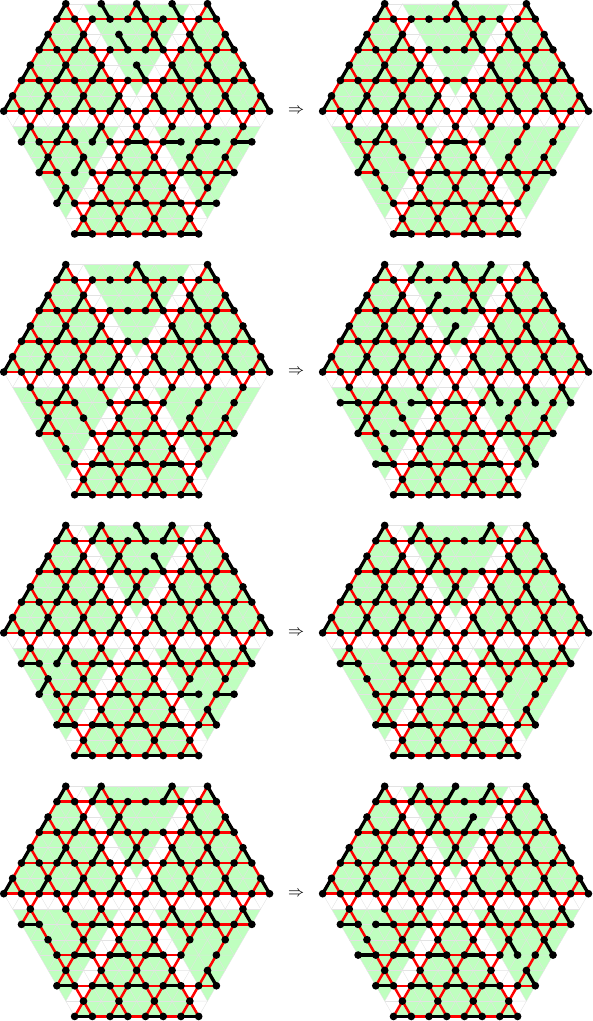}
    \caption{Switching primitive with higher expansion.}
    \label{fig:meta:hexagonal:rearrangement:alt}
\end{figure}

\begin{lemma}
\label{lem:meta:hexagonal:rearrangement:alt}
    Our alternative implementation of the switching primitive requires $4$ rounds.
\end{lemma}



\section{Reconfiguration}
\label{sec:reconfiguration}

In this section, we discuss possible reconfiguration algorithms.
For that, we look at reconfiguration algorithms for other lattice-type MRSs and discuss whether amoebots are capable of simulating these.
In particular, we consider the polylogarithmic time solutions for the nubot model \cite{DBLP:conf/innovations/WoodsCGDWY13} and the crystalline atom model \cite{DBLP:conf/isaac/AloupisCDLAW08}.
The first subsection deals with the former model while the second subsection deals with reconfiguration algorithms for our meta-modules including the one for the latter model.
None of the reconfiguration algorithms discussed in this section will be stable.

\subsection{Nubot Model}
\label{sec:nubot}
\label{app:nubot}

Similar to the amoebot model, the \emph{nubot model} \cite{DBLP:conf/innovations/WoodsCGDWY13} considers robots on the triangular grid where at most one robot can be positioned on each node.
Adjacent robots can be connected by rigid bonds\footnote{Woods et al.~\cite{DBLP:conf/innovations/WoodsCGDWY13} distinguish between rigid and flexible bonds. We ignore that since the flexible bonds are not necessary to achieve the polylogarithmic time reconfiguration algorithm.}.
In the joint movement extension, the rigid bonds correspond to the edge set $B$.

Robots are able to appear, disappear, and rotate around adjacent robots.
A rotating robot may push and pull other robots into its movement direction.
Hence, each rotation results in a translation of a set of connected robots (including the rotating robot) into the movement direction by the distance of 1.
The set depends on the bonds and the movement direction, and may include robots not connected by rigid bonds to the rotating robot.
In contrast to the joint movement extension, there are no collisions by definition of the set.
However, a movement may not be performed due to structural conflicts.
In this case, the nubot model does not perform the movement.

\begin{lemma}
    An amoebot structure of contracted amoebots is able to translate a set of robots in a constant number of joint movements if the translation is possible in the nubot model.
\end{lemma}

\begin{proof}
    Let $d$ denote the movement direction, $d'$ the opposite direction, and $d''$ any other cardinal direction.
    Let $M$ denote the set of amoebots that has to be moved.
    Instead of moving $M$ into direction $d$, we will move the remaining amoebot structure into direction $d'$.
    Note that $M$ divides the remaining amoebot structure into connected components.

    \begin{claim}
    \label{lem:nubot}
        Each node $x$ in direction $d'$ of an amoebot not in $M$ is either occupied by an amoebot of the same connected component or unoccupied.
    \end{claim}

    \begin{claimproof}
        Trivially, $x$ cannot be occupied by an amoebot of another connected component.
        Further, $x$ cannot be occupied by an amoebot in $M$ since otherwise, the resulting amoebot structure would not be free of collisions.
    \end{claimproof}

    For each connected component $C$, we perform the following steps in parallel.
    If there is an amoebot $u \in C$ that is adjacent to the same amoebot $v \in M$ before and after the translation,
    we proceed as follows (see \Cref{fig:nubot:1}).
    Let $A$ denote the row of amoebots in $C$ through $u$ into direction $d''$.
    Note that $A$ may only contain $u$.
    In the first round, we remove all bonds between $C$ and $M$ except for the bond between $u$ and $v$,
    and each amoebot in $A$ expands into direction $d'$.
    In the second round, we remove all bonds between $C$ and $M$ except for the new bond between $u$ and $v$,
    and each amoebot in $A$ contracts.
    Note that both movements are possible due to \Cref{lem:nubot}.

    If there is no amoebot in $C$ that is adjacent to the same amoebot in $M$ before and after the translation,
    we proceed as follows (see \Cref{fig:nubot:2}).
    Let $B$ denote all amoebots in $C$ that have an unoccupied node in direction $d'$.
    In the first round, each amoebot in $B$ expands into direction $d'$.
    There has to be an amoebot $u \in B$ that becomes adjacent to an amoebot $v \in M$.
    Otherwise, the resulting amoebot structure would not be connected.
    In the second round, we remove all bonds between $C$ and $M$ except for the bond between $u$ and $v$,
    and each amoebot in $B$ contracts.
    Note that both movements are possible due to \Cref{lem:nubot}.
\end{proof}

\begin{figure}[tb]
    \centering
    \begin{subfigure}[b]{\textwidth}
        \centering
        \includegraphics[page=1]{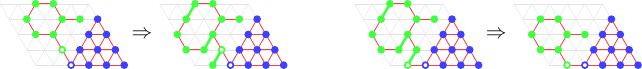}
        \caption{Amoebot $u \in C$ is adjacent to the same amoebot $v \in M$ before and after the translation. \smallskip}
        \label{fig:nubot:1}
    \end{subfigure}
    \begin{subfigure}[b]{\textwidth}
        \centering
        \includegraphics[page=1]{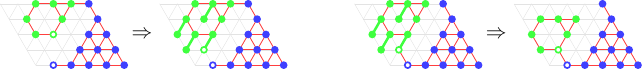}
        \caption{No amoebot in $C$ is adjacent to the same amoebot in $M$ before and after the translation.}
        \label{fig:nubot:2}
    \end{subfigure}
    \caption{
        Simulation of the nubot model.
        Let $d$ denote the northeastern direction, $d'$ the southwestern direction, and $d''$ the northwestern direction.
        The blue amoebots indicate set $M$.
        The green amoebots indicate set $C$.
        The amoebots marked by a white dot indicate amoebots $u$ and $v$.
    }
    \label{fig:nubot}
\end{figure}

Woods et al.~\cite{DBLP:conf/innovations/WoodsCGDWY13} showed that in the nubot model, the robots are able to self-assemble arbitrary shapes/patterns in an amount of time equal to the worst-case running time for a Turing machine to compute a pixel in the shape/pattern plus an additional factor which is polylogarithmic in its size.
While we are able to perform the translations, we do not have the means to let amoebots appear and disappear in the amoebot model.
This prevents us from simulating the reconfiguration algorithm by Woods et al.~\cite{DBLP:conf/innovations/WoodsCGDWY13}.

\subsection{Reconfiguration Algorithms for Meta-Modules}

Naturally, our meta-modules allow us to simulate reconfiguration algorithms for lattice-type MRSs of similar shape if we can implement the same movement primitives.
This leads us to the following results.

\begin{theorem}
\label{th:formation:rhombical}
    There is a centralized reconfiguration algorithm for $m$ rhombical meta-modules that requires $O(\log m)$ rounds and performs $\Theta(m \log m)$ moves overall.
\end{theorem}

\begin{proof}
    Aloupis et al.\ \cite{DBLP:conf/isaac/AloupisCDLAW08} proposed a reconfiguration algorithm for crystalline atoms.
    It requires $O(\log m)$ rounds and performs $\Theta(m \log m)$ moves overall.
    The idea is to transform the initial structure to a canonical structure using a divide and conquer approach.
    The target structure is reached by reversing that procedure.
    We refer to \cite{DBLP:conf/isaac/AloupisCDLAW08} for the details.
    The algorithm utilizes the contraction, slide and tunnel primitives which our rhombical meta-modules are capable of (see \Cref{sec:meta:rhombical}).
    Hence, they can simulate this reconfiguration algorithm.
\end{proof}

\begin{theorem}
\label{th:formation:hexagonal}
    There is a centralized reconfiguration algorithm for $m$ hexagonal meta-modules that requires $O(m)$ rounds. Each module has to perform at most $O(m)$ moves.
\end{theorem}

\begin{proof}
    Hurtado et al.\ \cite{DBLP:journals/arobots/HurtadoMRA15} proposed a reconfiguration algorithm for hexagonal robots.
    It requires $O(m)$ rounds and each module has to perform at most $O(m)$ moves.
    The idea is to compute a spanning tree and to move the robots along the boundary of the tree to a leader module where the robots form a canonical structure, e.g., a line.
    The target structure is reached by reversing that procedure.
    We refer to \cite{DBLP:journals/arobots/HurtadoMRA15} for the details.
    The algorithm utilizes the rotation primitive which our hexagonal meta-modules are capable of (see \Cref{sec:meta:hexagonal}).
    Hence, they can simulate this reconfiguration algorithm.
\end{proof}


\section{Locomotion}
\label{sec:motion}

In this section, we consider amoebot structures capable of locomotion along an even surface.
There are three basic types of terrestrial locomotion: rolling, crawling, and walking \cite{hirose1991three,DBLP:conf/embc/KehoeP19}.
We can find biological and artificial examples for each of those.
In the following subsections, we will present an amoebot structure for each type and analyze their velocity.
In the last subsection, we discuss the transportation of objects.
All constructions are locally and globally stable unless stated otherwise.


\subsection{Rolling}
\label{sec:motion:rolling}


Animals and robots that move by rolling either rotate their whole body or parts of it.
Rolling is rather rare in nature.
Among others, spiders, caterpillars, and shrimps are known to utilize rolling as a secondary form of locomotion during danger \cite{armour2006rolling}.
Bacterial flagella are an example for a creature that rotates a part of its body around an axle \cite{labarbera1983wheels}.

In contrast, rolling is commonly used in robotic systems mainly in the form of wheels.
An example of a robot system rolling as a whole are chain-type MRSs that can roll by forming a loop, e.g., Polypod \cite{DBLP:conf/icra/Yim94}, Polybot \cite{DBLP:conf/icra/YimDR00,DBLP:journals/arobots/YimRDZEH03}, CKBot \cite{DBLP:conf/icra/MellingerKY09,DBLP:journals/ijrr/SastraCY09}, M-TRAN \cite{DBLP:conf/cira/YoshidaMKTKK03,DBLP:journals/ras/KurokawaYTKMK06}, and SMORES \cite{DBLP:conf/rss/JingTYK16,DBLP:journals/arobots/JingTYK18}.
Further, examples for rolling robots can be found in \cite{armour2006rolling,brunete2017current}.


Our rolling amoebot structure imitates a continuous track that rotates around a set of wheels.
Continuous tracks are deployed in various fields, e.g., construction, agriculture, and military.
We build our \emph{continuous track structure} from hexagonal meta-modules of alternating side lengths $\ell$ and $\ell - 1$ (see \Cref{fig:motion:loop:simplified}).
The structure consists of two parts: a connected substrate structure (green meta-modules), and a closed chain of meta-modules rotating along the outer boundary of the substrate (blue meta-modules).

\begin{figure}[!tb]
    \centering
    \includegraphics[page=1]{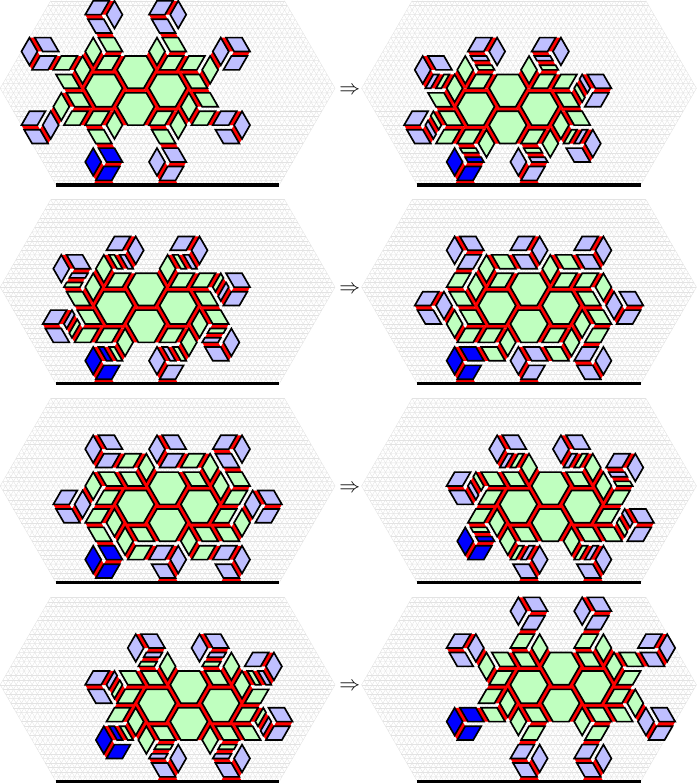}
    \caption{
        Continuous track structure.
        The blue meta-modules rotate clockwise around the green meta-modules.
        We highlight one of the rotating meta-modules in a darker blue.
    }
    \label{fig:motion:loop:structure}
    \label{fig:motion:loop:simplified}
\end{figure}


The continuous track structure moves as follows.
The rotating meta-modules that are in contact with the surface release all such bonds with the surface if they rotate away from the surface (see the dark blue meta-module in the third round).
Otherwise, they keep these bonds such that the substrate structure is pushed forwards.
Note that we have to apply the switching primitive between the rotations.
We obtain the initial structure after two rotations.
In doing so, the structure has moved a distance of $2 \cdot \ell$.
By performing the movements periodically, we obtain the following theorem.

\begin{theorem}
    Our continuous track structure composed of hexagonal meta-modules of alternating side lengths $\ell$ and $\ell - 1$ moves a distance of $2 \cdot \ell$ within each period of constant length.
\end{theorem}



Butler et al.\ \cite{DBLP:journals/ijrr/ButlerKRT04} have proposed another rolling structure for the sliding-cube model that we are able to simulate.
The structure resembles a swarm of caterpillars where caterpillars climb over each other from the back to the front \cite{dolev2016vivo}.
However, due to stalling times, this structure is slower than our continuous track structure.
We refer to \cite{DBLP:journals/ijrr/ButlerKRT04} for the details.



\subsection{Crawling}
\label{sec:motion:crawling}


Crawling locomotion is used by limbless animals.
According to \cite{juhasz2013analysis}, crawling can be classified into three types: worm-like locomotion, caterpillar-like locomotion, and snake-like locomotion.
We will explain the earthworm-like locomotion below and refer to \cite{juhasz2013analysis} for the other two types.
Due to the advantage of crawling in narrow spaces, various crawling structures have been developed for MRSs, e.g., crystalline atoms \cite{DBLP:conf/iser/KotayRV00,DBLP:conf/icra/RusV99,DBLP:journals/arobots/RusV01}, catoms \cite{DBLP:journals/ijrr/CampbellP08,DBLP:conf/icra/ChristensenC07}, polypod \cite{DBLP:conf/icra/Yim94}, polybot \cite{DBLP:journals/arobots/YimRDZEH03,DBLP:conf/cira/ZhangYEDR03}, M-TRAN \cite{DBLP:journals/ras/KurokawaYTKMK06,DBLP:conf/cira/YoshidaMKTKK03}, and origami robots \cite{DBLP:journals/robotics/KalairajCSR21}.


Our crawling amoebot structure imitates earthworms.
An earthworm is divided into a series of segments.
It can individually contract and expand each of its segments.
Earthworms move by peristaltic crawling,
i.e., they propagate alternating waves of contractions and expansions of their segments from the anterior to the posterior part.
The friction between the contracted segments and the surface gives the worm grip.
This anchors the worm as other segments expand or contract.
The waves of contractions pull the posterior parts to the front,
and the waves of expansions push the anterior parts to the front.
\cite{juhasz2013analysis,DBLP:conf/iros/OmoriHN08}


The simplest amoebot structure that imitates an earthworm is a simple line of $n$ expanded amoebots along the surface (see \Cref{fig:motion:worm:simple}).
Each amoebot can be seen as a segment of the worm.
Instead of propagating waves of contractions and expansions,
we contract and expand the whole structure at once.
During each contraction (expansion), we release all bonds between the amoebot structure and the surface except for the ones at the head (tail) of the structure that serve as an anchor.
As a result, the contraction (expansion) pulls (pushes) the structure to the front.
The simple line has moved by a distance of $n$ along the surface after performing a contraction and an expansion.
This is the fastest way possible to move along a surface
since we accumulate the movements of all amoebots into the same direction.
By performing the movements periodically, we obtain the following theorem.

\begin{figure}[tb]
    \centering
    \includegraphics[page=1]{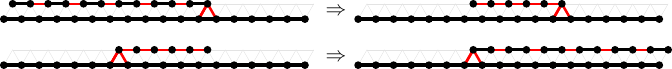}
    \caption{
        Simple line of amoebots.
    }
    \label{fig:motion:worm:simple}
\end{figure}

\begin{theorem}
    A simple line of $n$ expanded amoebots moves a distance of $n$ every $2$ rounds.
\end{theorem}



However, in practice, the contractions and expansions of the whole structure yield high forces acting on the connections within the amoebot structure.
We can address this problem by thickening the worm structure.
This increases the expansion of the structure and with that its stability.
Consider a line of $r$ rhombical meta-modules of side length $\ell - 1$ (see Figure~\ref{fig:motion:worm:simplified}).
Each module can be seen as a segment of the worm structure.
Recall that we can contract a rhombical meta-module into a parallelogram (see \Cref{fig:meta:rhombical:contraction}).
The line of rhombical meta-modules moves in the same manner as the simple line except for the following two points.
First, we utilize the whole meta-module at the front and at the end as an anchor to increase the grip, respectively.
Second, only the middle $r-2$ meta-modules participate in the contractions and expansions.
The line of rhombical meta-modules moves a distance of $\frac{r-2}{2} \cdot \ell$ along the surface after performing a contraction and an expansion.
By performing the movements periodically, we obtain the following theorem.

\begin{figure}[tb]
    \centering
    \includegraphics[page=1]{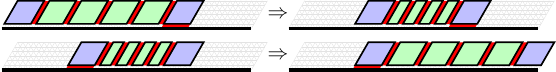}
    \caption{
        Line of rhombical meta-modules.
    }
    \label{fig:motion:worm:simplified}
\end{figure}

\begin{theorem}
    A line of $r$ rhombical meta-modules of side length $\ell - 1$ moves a distance of $\frac{r-2}{2} \cdot \ell$ every $2$ rounds.
\end{theorem}



Note that unless the number of amoebots resp.\ meta-modules is constant, neither amoebot structure in the subsection is locally stable since we have to remove almost all bonds between the structure and the surface.
Another problem in practice that comes with the local instability is friction between the structure and the surface and with that the wear of the structure.
The worm structure is therefore poorly scalable in its length such that other types for locomotion are more suitable for large amounts of amoebots.


Most of the cited MRSs at the beginning of this section are very similar to our construction.
The construction for crystalline robots is the closest one.
Each of these consists of a line of (meta-)modules that are able to contract.
Katoy et al.~\cite{DBLP:conf/iser/KotayRV00} also propose a `walking' structure (see \Cref{fig:motion:inchworm:simplified}).
However, the locomotion is still caused by the contraction of the body instead of motions of the legs.
So, it is rather a caterpillar-like crawling than a walking movement.

\begin{figure}[tb]
    \centering
    \includegraphics[page=1]{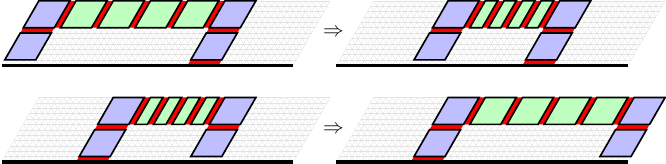}
    \caption{
        Caterpillar structure.
        This structure is a replication of the structure by Katoy et al.\ \cite{DBLP:conf/iser/KotayRV00} in the joint movement extension.
    }
    \label{fig:motion:inchworm:simplified}
\end{figure}



\subsection{Walking}
\label{sec:motion:walking}


A wide variety of animals are capable of walking locomotion, e.g., mammals, reptiles, birds, insects, millipedes, and spiders.
Just as wide is the variety of differences, e.g., they differ in the number of legs, in the structure of the legs, and in their gait.
Walking structures have been built for chain-type MRSs, e.g., M-TRAN \cite{DBLP:conf/cira/YoshidaMKTKK03,DBLP:journals/ras/KurokawaYTKMK06} and polybot \cite{DBLP:journals/arobots/YimRDZEH03,DBLP:conf/cira/ZhangYEDR03}.


Our walking amoebot structure imitates millipedes.
Millipedes have flexible, segmented bodies with tens to hundreds of legs that provide morphological robustness \cite{DBLP:journals/ral/ShaoDLTSLZ22}.
They move by propagating leg-density waves from the posterior to the anterior \cite{Garcia_2021,Kuroda_2022}.
We build our \emph{millipede structure} from rhombical meta-modules of side length $\ell - 1$ (see \Cref{fig:motion:millipede:simplified}).
Let $p$ denote the number of legs.
The body and each leg consists of a line of rhombical meta-modules.
All legs have the same size.
Let $q$ denote the number of rhombical modules in each leg.
We attach each leg to a meta-module of the body and orientate them alternating to the front and to the back.
We call those meta-modules \emph{connectors}.
In order to prevent the legs from colliding, we place $q$ meta-modules between two connectors.
Altogether, the millipede structure consists of $(2p - 1) \cdot q + p$ meta-modules.

\begin{figure}[tb]
    \centering
    \includegraphics[page=1]{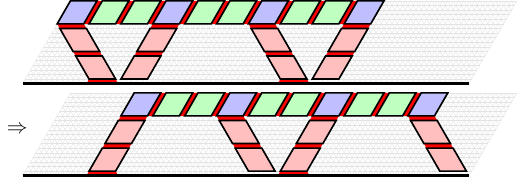}
    \caption{
        Millipede structure.
        The arrow indicates two rounds.
        The red rhombi indicate the legs.
        The blue and green rhombi indicate the body.
        The blue rhombi indicate the connectors.
        Note that the figure only depicts the first half of a period.
    }
    \label{fig:motion:millipede:simplified}
\end{figure}


In order to move the legs back and forth, we simply apply the realignment primitive (see \Cref{fig:meta:rhombical:realignment}) on all meta-modules within the legs (see \Cref{fig:motion:millipede:simplified}).
We achieve forward motion by releasing all bonds between the surface and the legs moving forwards.
In one step, the body moves a distance of $q \cdot \ell$.
Note that the number of legs has no impact on the velocity.
By continuously repeating these leg movements, we achieve a motion similar to the leg-density waves of millipedes.
Note that we reach the initial amoebot structure after two leg movements.
Hence, we obtain the following theorem.

\begin{theorem}
    Our millipede structure composed of rhombical meta-modules of side length $\ell - 1$ with $p$ legs composed of $q$ rhombical meta-modules moves a distance of $2 \cdot q \cdot \ell$ every $4$ rounds.
\end{theorem}


In practice, we can reduce the friction by additionally lifting the legs moving forwards (see \Cref{fig:motion:millipede}).
For that, it suffices to partially contract all meta-modules of the body except for the connectors connected to a leg moving backwards.
After the movement, we lower the lifted legs back to the ground.
For that, we reverse the contractions within the body.

\begin{figure}[!tb]
    \centering
    \includegraphics[page=1]{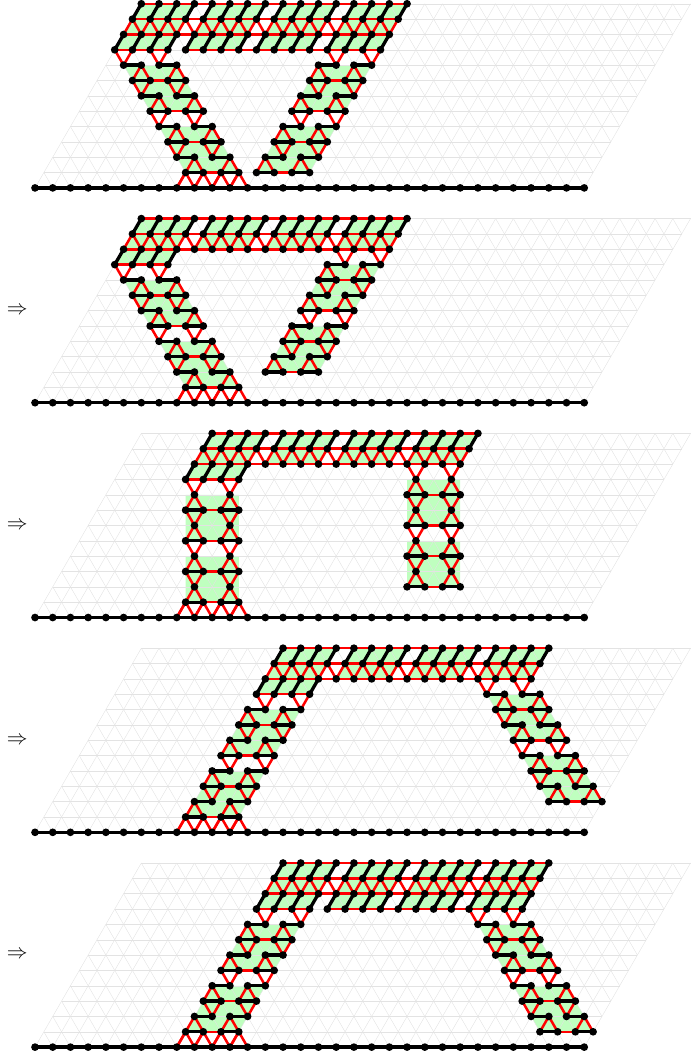}
    \caption{
        Movement of a millipede.
        For the sake of simplicity, we only show two legs.
    }
    \label{fig:motion:millipede}
\end{figure}



\subsection{Transportation}
\label{app:transportation}

Another important aspect of locomotion is the transportation of objects.
The continuous track and worm structure are unsuitable for the transportation of objects due to their unstable top.
In the worm structure, we can circumvent that problem by increasing the number of meta-modules not participating in the contractions and expansions.
However, this decreases the velocity of the worm structure and increases the friction under the object.

In contrast, the millipede structure provides a rigid transport surface for the transportation of objects (see \Cref{fig:motion:millipede:transportation}).
In practice, a high load introduced by the object on the structure may lead to problems.
However, it was shown that millipedes have a large payload-to-weight ratio since they have to withstand high loads while burrowing in leaf litter, dead wood, or soil \cite{Garcia_2021,DBLP:journals/ral/ShaoDLTSLZ22}.
The millipede is able to distribute the weight evenly on its legs.
Consequently, the millipede structure is well suited for the transportation of objects.

\begin{figure}[tb]
    \centering
    \includegraphics[page=1]{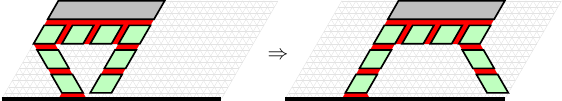}
    \caption{
        Transportation of objects by the millipede structure.
        The arrow indicates two rounds.
        The green rhombi indicate the millipede structure.
        The gray parallelogram indicates the object.
    }
    \label{fig:motion:millipede:transportation}
\end{figure}



\section{Conclusion and Future Work}
\label{sec:conclusion}

In this paper, we have formalized the joint movement extension that were proposed by Feldmann et al.~\cite{DBLP:journals/jcb/FeldmannPSD22}.
We have constructed meta-modules of rhombical and hexagonal shape that are able to perform various movement primitives.
This allows us to simulate reconfiguration algorithms of various MRSs.
However, our meta-modules are more flexible, e.g., we can move a hexagonal meta-module through two others (see \Cref{app:meta:hexagonal:wall}).

Such new movement primitives may lead to faster reconfiguration algorithms, e.g., sublinear solutions for hexagonal meta-modules or even arbitrary amoebot structures.
Furthermore, we have presented three amoebot structures capable of moving along an even surface.
In future work, movement on uneven surfaces can be considered.




\bibliography{literature}

\appendix

\clearpage

\section{Moving Meta-Modules through Others}
\label{app:meta:hexagonal:wall}

In this section, we show how to move a hexagonal meta-module $H_1$ through two other hexagonal meta-modules $H_2$ and $H_3$ (see \Cref{fig:meta:hexagonal:wall:2}).
We first form six rhombical meta-modules from $H_1$ and the amoebots of $H_2$ and $H_3$ that are located between the old and new position of $H_1$.
For that, we shift the amoebots along the red lines.
In order to shift the amoebots along the red lines, we first contract all amoebots on the line and then expand all amoebots on the line into the other direction.
Using the realignment primitive and the slide primitive for rhombical meta-modules, we shift the six rhombical meta-modules into the new position of $H_1$.
Finally, we restore $H_2$ and $H_3$ by shifting the amoebots along the red lines.
Note that we have omitted some necessary reorientations.

\begin{lemma}
    We can move a hexagonal meta-module through two other hexagonal meta-modules within $O(1)$ rounds.
\end{lemma}

\begin{figure}[!tb]
    \centering
    \includegraphics[page=1]{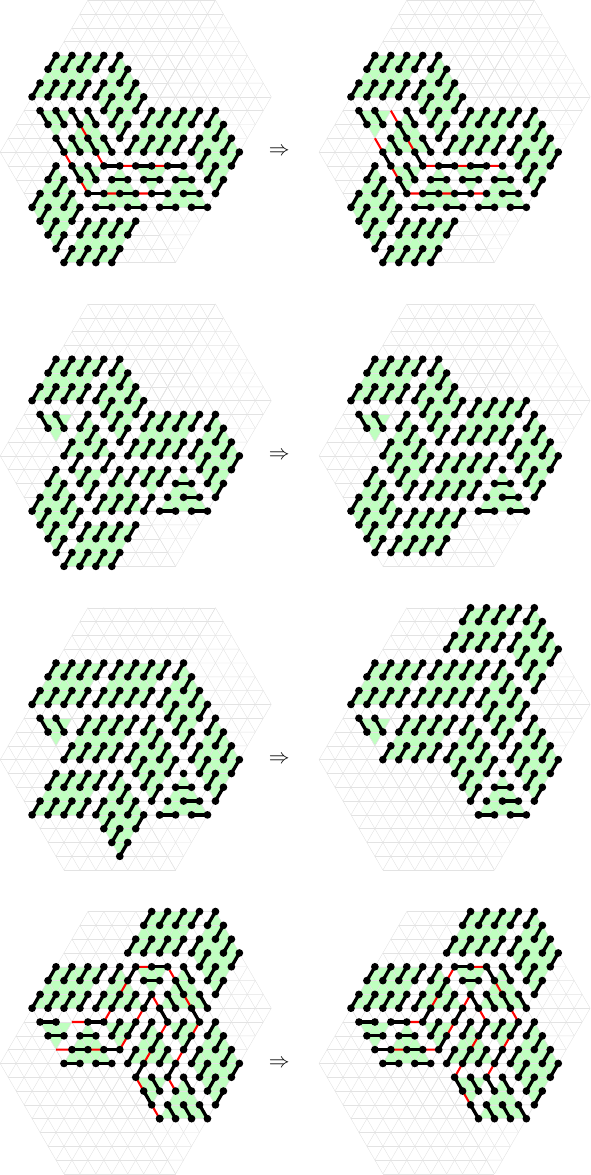}
    \caption{
        Moving a hexagonal meta-module through two other hexagonal meta-modules.
        The arrows indicate multiple rounds, respectively.
        For the sake of clarity, we have omitted the bonds in these figures.
        In the first and last step, the amoebots move along the red lines.
    }
    \label{fig:meta:hexagonal:wall:2}
\end{figure}


\end{document}